\documentclass{article}

\usepackage{microtype}
\usepackage{graphicx}
\usepackage{paralist}
\usepackage{multirow}
\usepackage{subcaption}
\usepackage{float}
\usepackage{booktabs} %
\usepackage[dvipsnames]{xcolor}

\usepackage[backref=page]{hyperref}

\usepackage[accepted]{icml2024}

\usepackage{amsmath}
\usepackage{amssymb}
\usepackage{mathtools}
\usepackage{amsthm}

\usepackage[capitalize,noabbrev]{cleveref}

\theoremstyle{plain}
\newtheorem{theorem}{Theorem}[section]
\newtheorem{proposition}[theorem]{Proposition}
\newtheorem{lemma}[theorem]{Lemma}
\newtheorem{corollary}[theorem]{Corollary}
\theoremstyle{definition}
\newtheorem{definition}[theorem]{Definition}

\theoremstyle{remark}

\newcommand{\data}{\mathcal{D}}
\newcommand{\E}{\mathbb{E}}
\newcommand{\normdist}{\mathcal{N}}
\newcommand{\loss}{\mathcal{L}}
\newcommand{\KL}{\text{KL}}

\newcommand{\VI}{\text{VI}}
\newcommand{\X}{\mathbf{X}}
\newcommand{\Y}{\mathbf{Y}}
\newcommand{\x}{\mathbf{x}}
\newcommand{\y}{\mathbf{y}}
\newcommand{\bb}{\mathbf{b}}
\newcommand{\f}{\mathbf{f}}
\newcommand{\g}{g}
\newcommand{\w}{\boldsymbol{\omega}}
\newcommand{\WW}{\mathbf{W}}
\newcommand{\W}{\mathcal{W}}
\newcommand{\V}{\mathcal{V}}
\newcommand{\m}{\boldsymbol{\mu}}
\newcommand{\s}{\boldsymbol{\Sigma}}
\newcommand{\vol}{\text{vol}}
\newcommand{\diag}{\text{diag}}

\usepackage[disable,textsize=tiny]{todonotes}

\icmltitlerunning{Permutation Invariant Variational Posteriors for Bayesian Neural Networks}

\begin{document}

\twocolumn[
\icmltitle{Variational Inference Failures Under Model Symmetries: Permutation Invariant Posteriors for Bayesian Neural Networks}

\icmlsetsymbol{equal}{*}

\begin{icmlauthorlist}
\icmlauthor{Yoav Gelberg}{oxford}
\icmlauthor{Tycho F.A. van der Ouderaa}{imperial,oxford}
\icmlauthor{Mark van der Wilk}{oxford}
\icmlauthor{Yarin Gal}{oxford}
\end{icmlauthorlist}

\icmlaffiliation{oxford}{University of Oxford, Oxford, UK}
\icmlaffiliation{imperial}{Imperial College London, London, UK}

\icmlcorrespondingauthor{Yoav Gelberg}{yoav@robots.ox.ac.uk}

\icmlkeywords{Bayesian Deep Learning, Variational Inference, Geometric Deep Learning, Group Invariance}

\vskip 0.3in
]

\printAffiliationsAndNotice{}  %

\begin{abstract}
Weight space symmetries in neural network architectures, such as permutation symmetries in MLPs, give rise to Bayesian neural network (BNN) posteriors with many equivalent modes. This multimodality poses a challenge for variational inference (VI) techniques, which typically rely on approximating the posterior with a unimodal distribution. In this work, we investigate the impact of weight space permutation symmetries on VI. We demonstrate, both theoretically and empirically, that these symmetries lead to biases in the approximate posterior, which degrade predictive performance and posterior fit if not explicitly accounted for. To mitigate this behavior, we leverage the symmetric structure of the posterior and devise a \textit{symmetrization} mechanism for constructing permutation invariant variational posteriors. We show that the symmetrized distribution has a strictly better fit to the true posterior, and that it can be trained using the original ELBO objective with a modified KL regularization term. We demonstrate experimentally that our approach mitigates the aforementioned biases and results in improved predictions and a higher ELBO.
\end{abstract}

\section{Introduction}
\looseness=-1
Bayesian neural networks (BNNs) \cite{118274,10.1145/168304.168306, Neal1995BayesianLF} model neural networks probabilistically by inferring a posterior distribution over their weights. This approach offers robustness to overfitting, epistemic uncertainty estimates, and the ability to learn from limited data, but is typically intractable due to the complexity of the posterior distribution. To use BNNs in practice, posterior estimates are obtained through approximate inference techniques such as variational inference (VI) \citep{blundell2015weight}.

\looseness=-1
While VI facilitates practical posterior estimation, it often relies on unimodal variational posterior families, limiting the accuracy of approximating the typically highly multimodal posterior. Multimodality of the posterior arises from various sources, including the data, task, and model choice. In BNNs specifically, inherent redundancies in the parametric representation of the model, caused by architectural symmetries \cite{HECHTNIELSEN1990129, Pourzanjani2017ImprovingTI, Rossi2023OnPS, laurent2024a}, are key sources of posterior multimodality. This phenomenon, sometimes referred to as non-identifiability, leads to the emergence of equivalent modes in the posterior that represent networks with identical functionality but different weight configurations.

\begin{figure}[t]
\centering
\begin{tabular}{cc}
  \includegraphics[width=0.45\linewidth]{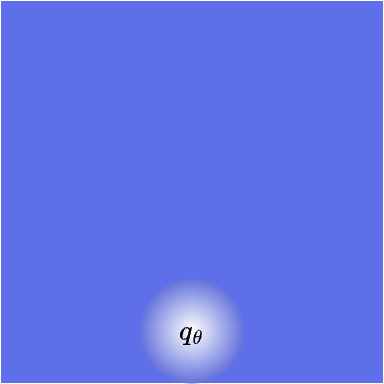} &
  \includegraphics[width=0.45\linewidth]{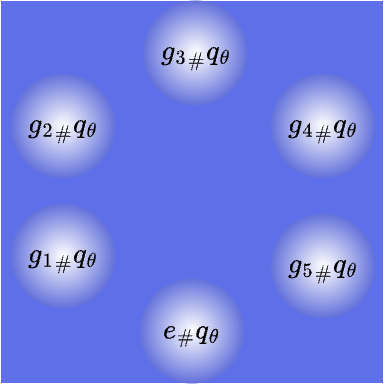}
\end{tabular}
\caption{\textbf{Symmetrization of the variational posterior.} Given a variational distribution $q_\theta(\w)$, its symmetrization with respect to a group $G$ acting on the underlying space is the average of the pushforwards of $q_\theta(\w)$ over all group elements $g \in G$.}
\label{fig:symmterization_of_the_posterior}
\end{figure}

\looseness=-1
This work investigates the behavior of unimodal VI approximations in the presence of equivalent modes. We show that unimodal variational distributions tend to interpolate between closely spaced modes, causing a shift in mean and variance. In the case of equivalent modes of the posterior, interpolation results in worse predictive performance and contributes to the underfitting tendencies of variational BNNs \cite{Wenzel2020AppendixFH, pmlr-v80-ghosh18a, Dusenberry2020EfficientAS}. We demonstrate this effect in a VI setup with a Gaussian mixture target distribution and in a simple BNN with a tractable posterior. Additionally, we provide theoretical analysis on the spacing of equivalent modes in BNN posteriors, suggesting the severity of the problem grows with model width.

\vspace{-1mm}
\looseness=-1
To address these limitations, we leverage the symmetric structure of the posterior and propose an approach that ``bakes in'' permutation invariance into the variational distribution. We devise a generic symmetrization mechanism that averages a black-box variational posterior family over the orbits of the permutation symmetry group. This mechanism results in a permutation invariant variational posterior which has a \textit{provably} better fit to the true posterior. We show that training the symmetrized distribution can be done by modifying the prior KL term in the ELBO associated with the original distribution. Evaluating this new KL term requires computing the entropy of a symmetric mixture of densities, for which we derive a computationally efficient estimator. This is significant as it eliminates the need for computationally intractable averaging over \textit{all} weight-space permutations. Instead, symmetry is accounted for by stochastically adapting the KL regularization term. We conduct empirical evaluations comparing our symmetrization method to standard mean-field VI. We demonstrate that our approach achieves better predictive performance and a higher ELBO.

\vspace{-1mm}
\looseness=-1
\textbf{Contributions.} This paper makes the following contributions: \begin{inparaenum}[(1)] \item It demonstrates that in the presence of weight space symmetries, the use of unimodal approximate posterior distributions may result in interpolation between equivalent modes, which degrades predictive performance and posterior fit. \item It devises a symmetrization procedure for approximate posterior families, proves that it strictly improves posterior fit, and derives a tractable ELBO objective for it by constructing a novel entropy estimator for symmetric mixture of densities. \item It demonstrates, experimentally, that VI using symmetrized variational posteriors prevents equivalent mode interpolation and results in better predictions and posterior fit, improving overall performance over mean-field VI.\end{inparaenum}

\looseness=-1
\textbf{Related work.} Several studies have looked into the effects of non-identifiability on approximate inference in BNNs. \citet{kurle2022on} formally discussed the effects of invariances of the likelihood function on mean-field VI and proposed a mitigation for data-dependent translation symmetries in Bayesian linear regression and BNNs. \citet{Pourzanjani2017ImprovingTI} and \citet{Kurle2021} suggested constraining the model architecture to lower the number of weight space symmetries. Other works proposed approximating the posterior predictive directly \cite{sun2018functional, rudner2022fsvi}, thereby performing VI in function space. In this work, we focus on the effects of permutation symmetries, and mitigate the problem without constraining the architecture. Instead, we construct a permutation invariant variational posterior and derive an estimator for its VI objective. To the best of our knowledge, this is the first work to develop \textit{trainable} permutation invariant variational posteriors for BNNs.

\section{Preliminaries} \label{sec:perlimiaries}
\looseness=-1
\textbf{Notation.} We use bold lower case letters (e.g. $\x$) for vectors, bold upper case letters (e.g. $\X$) for matrices, and standard letters (e.g. $x$) for scalars. $g$ is reserved for group elements and $e$ denotes the identity. $\w \in \W$ represents model parameters and $\f^{\w}$ denotes the function parameterized by $\w$. $\W$ is referred to as the weight space of $\f^{\w}$. Throughout the paper we assume all discussed distributions admit a density function, and adopt the convention of identifying probability distributions with their densities (e.g. $p$ denotes both a probability measure $p(\x \in A)$ and its density $p(\x)$).

\looseness=-1
\textbf{Bayesian deep learning.} Let $\data = (\X, \Y)$ be a dataset, where $\X = (\x_1, \dots, \x_N)$ represents the inputs and $\Y = (\y_1, \dots, \y_N)$ represents the corresponding outputs. Bayesian deep learning seeks parameters $\w \in \W$ of a function $\y = \f^{\w}(\x)$ that are most \textit{likely} to have generated the observed outputs $\Y$. To do so, we place a prior distribution $p(\w)$ over $\W$ which encodes our initial beliefs about the likelihood of the parameters before observing any data. Upon observing the data, the prior is updated to the posterior distribution $p(\w \mid \data)$ using Bayes' theorem
\begin{equation}\label{eq:bayes}
p(\w \mid \data) = \frac{p(\Y \mid \X, \w) p(\w)}{p(\Y \mid \X)}
\end{equation}
\looseness=-1
Here, $p(\Y \mid \X, \w)$ denotes the likelihood, which describes the probability of observing the outputs $\Y$ given the inputs $\X$ and parameters $\w$. Assuming i.i.d data, $p(\Y \mid \X, \w) = \prod_{i = 1}^N p(\y_i \mid \x_i, \w) = \prod_{i = 1}^N p(\y_i \mid \f^{\w}(\x_i))$. The choice of the likelihood function $p(\y \mid \f^{\w}(\x))$ depends on the task:
\begin{enumerate}
    \item For classification tasks, a common choice is the softmax likelihood $p(y = k \mid\x, \w) = \text{softmax} (\f^{\w} (\x))_k$.
    \item In regression tasks, likelihood is usually assumed to be Gaussian, i.e. $p(\y \mid \x, \w) = \mathcal{N}(\y ; \f^{\w} (\x), \sigma^{-1} \mathbf{I})$.
\end{enumerate}

\looseness=-1
A key component in Equation \eqref{eq:bayes} is the normalizing constant, $p(\Y \mid \X)$, also known as the \textit{model evidence} or \textit{marginal likelihood}. $p(\Y \mid \X)$ integrates over all possible parameter values according to the prior
\begin{equation}\label{eq:marginal_likelihood}
p(\Y \mid \X) = \int p(\Y \mid \X, \w) p(\w) d\w
\end{equation}
\looseness=-1
While analytical computation of the marginal likelihood (and thus the posterior) is possible for simple models, it becomes intractable for complex models such as deep neural networks.

\looseness=-1
\textbf{Variational inference.} In cases where exact posterior evaluation is not feasible, variational inference provides an approach to approximate the true posterior $p(\w \mid \data)$ using a tractable variational distribution $q_\theta (\w)$. To find the best approximate posterior distribution, the Kullback–Leibler divergence (KL) between $q_\theta(\w)$ and the true posterior is minimized
\begin{equation*}
\begin{split}
    \KL\left(q_\theta (\w) \mid\mid p(\w \mid \data)\right) &= \int q_\theta(\w) \log\left(\frac{q_\theta (\w)}{p(\w \mid \data)}\right) d\w \\
    &= \E_{\w \sim q_\theta} \left( \log\left(\frac{q_\theta (\w)}{p(\w \mid \data)}\right) \right)
\end{split}
\end{equation*}
Since the true posterior is intractable for large models, the above KL  divergence is intractable in general. However, it turns out that minimizing KL divergence is equivalent to maximizing the evidence lower bound (ELBO), defined by
\begin{multline}\label{eq:elbo}
    \loss_{\text{VI}}(\theta) := \E_{\w \sim q_\theta(\w)} \left(\log\left(p(\Y \mid \X, \w)\right)\right)  \\ -\KL\left(q_\theta(\w) \mid \mid p(\w)\right) 
\end{multline}
The first term of the ELBO is referred to as the \textit{expected log-likelihood} term and the second term is referred to as the \textit{prior KL} term. Maximizing the expected log-likelihood term encourages $q_\theta(\w)$ to fit the data, while minimizing the prior KL term encourages $q_\theta(\w)$ to remain close to the prior. The term evidence lower bound comes from the identity
\begin{equation*}
\begin{split}
    \loss_\VI(\theta) &= \log\left(p\left(\Y \mid \X \right)\right) - \KL \left(q_\theta (\w) \mid \mid p(\w \mid \data)\right) \\
    &\leq \log(p(\Y \mid \X))
\end{split}
\end{equation*}
Therefore, since the marginal likelihood is constant w.r.t $\theta$, maximizing the ELBO is equivalent to minimizing KL with the posterior. For more details, see \cite{blundell2015weight}.

\textbf{Group representations.} Given a group $G$ and a vector space $\V$, a representation of $G$ is a homomorphism $\rho: G \to GL(\V)$. We call $\rho$ an \textit{orthonormal} representation if $\V$ admits an inner product and $\rho(G) \subseteq O(\V)$, i.e. if $\rho(g)$ are norm-preserving linear transformations. When $\rho$ is clear from the context, we identify $g$ with its representation $\rho(g)$ and write $g \cdot \mathbf{v} := \rho(g) \mathbf{v}$, $\det(g) := \det(\rho(g))$.

\textbf{Group invariant functions and measures.} A function $f: \V \to \mathcal{Y}$ is called $G$-invariant if $\forall g \in G, \mathbf{v} \in \V, \, f(g \cdot \mathbf{v}) = f(\mathbf{v})$. A measure $\nu$ on $\V$ is called $G$-invariant if $\forall g \in G$ and for every measurable set $A \subseteq \V$, $\nu \left(g^{-1} \cdot A\right) = \nu(A)$. The distribution defined by $g_\# \nu(A) := \nu(g^{-1} \cdot A)$ is called the $g$-pushforward measure. Under this definition, $\nu$ is $G$-invariant iff $g_\# \nu = \nu$ for all $g \in G$. If $\nu$ has a density, $f_\nu$, i.e. $\nu(A) = \int_A f_\nu(\mathbf{v}) \, d\mathbf{v}$, it transforms under pushforward and becomes $f_{g_\# \nu} (\mathbf{v}) = \frac{f_\nu(g^{-1} \cdot \mathbf{v})}{\lvert \det(g) \rvert}$. If the representation is orthonormal (thus $\det(g) = \pm 1$) then $f_{g_\# \nu} (\mathbf{v}) = f_\nu(g^{-1} \cdot \mathbf{v})$.

\textbf{Multilayer perceptrons.} Multilayer Perceptrons (MLPs) are feedforward neural networks with fully connected layers. Formally, an $L$-layer MLP is a parametric function $\f = \f^{\w}$ defined recursively by
\begin{equation*}
\f^{\w}(\x) = \x_L, \quad \x_{l + 1} = \sigma(\WW_{l + 1} \x_l + \textbf{b}_{l + 1}), \quad \x_0 = \x
\end{equation*}
where $\WW_l \in \mathbb{R}^{d_l \times d_{l -1}}$ and $\textbf{b}_l \in \mathbb{R}^{d_l}$. $\w = (\WW_1, \dots, \WW_L, \textbf{b}_1, \dots, \textbf{b}_L)$ represents the MLP's parameters referred to as weights. $d_1, \dots, d_{L - 1}$ are called the \textit{hidden dimensions}, $d_0$ is the \textit{input dimension}, and $d_L$ is the \textit{output dimension}. Note that the weight space of an MLP is a real vector space and can be identified with $\bigoplus_{l = 1}^L\left(\mathbb{R}^{d_l \times d_{l - 1}} \oplus \mathbb{R}^{d_l}\right)$.

\textbf{Permutation symmetries of MLPs.} MLPs exhibit well-documented permutation invariance properties \cite{HECHTNIELSEN1990129, Pourzanjani2017ImprovingTI,simsek2020weightspace, ainsworth2023git, DWS, Zhou2024UniversalNF}: permuting the neurons of any hidden layer, while keeping track of the connections to the neighboring layers, preserves the function represented by the MLP. More formally, this invariance can be described by a group representation. 
\begin{definition}[MLP permutation symmetry group]
    Given an MLP with hidden dimensions $d_1, \dots d_{L-1}$, its permutation symmetry group is
    \begin{equation*}
        G = S_{d_1} \times \dots \times S_{d_{L - 1}}
    \end{equation*}
\end{definition}
\looseness=-1
Given $\w  = (\WW_1, \dots, \WW_L, \textbf{b}_1, \dots, \textbf{b}_L) \in \W$ and $g = (\tau_1, \dots, \tau_{L - 1}) \in G$, $g$'s action on $\w$ is defined by $g \cdot \w = \w' = (\WW_1', \dots, \WW_L', \textbf{b}_1', \dots, \textbf{b}_L')$ with
\begin{align*}
    &\WW_1' = \mathbf{P}_{\tau_1}^\top \WW_1, & &\bb_1' = \mathbf{P}_{\tau_1}^\top \bb_1 \\ 
    &\WW_l' = \mathbf{P}_{\tau_l}^\top \WW_l \mathbf{P}_{\tau_{l -1}}, & &\bb_l' = \mathbf{P}_{\tau_l}^\top \bb_l, \quad l \in \{2, \dots, L -1\} \\
    &\WW_{L}' = \WW_L \mathbf{P}_{\tau_{L - 1}}, & &\bb_L' = \bb_L
\end{align*}
where $\mathbf{P}_{\tau_l} \in \mathbb{R}^{d_l \times d_l}$ is the permutation matrix associated with $\tau_l \in S_{d_l}$. It's straightforward to verify that the mapping $\w \mapsto g \cdot \w$ is linear and norm-preserving and that the action of $G$ defines an orthonormal representation on $\W$. Additionally, for any point-wise non-linearity, the transformed weights under $g$ represent the same function as the original weights, i.e. $\f^{g \cdot \w} \equiv \f^{\w}$. In other words, the mapping $\w \mapsto \f^{\w}$ is $G$-invariant.

\looseness=-1
\textbf{Permutation invariance of the posterior.} Throughout the paper, we assume a $G$-invariant prior on $\W$. This assumption is reasonable as apriori all weight configurations representing the same function are equally likely. Additionally, the most commonly used priors over BNNs (such as the isotropic Gaussian prior, the de-facto standard \cite{fortuin2022bayesian}), are invariant to weight space permutations. Since the likelihood $p(\y \mid \x, \w) = p(\y \mid \f^{\w}(\x))$ is a function of $\f^{\w}(\x)$, it too is $G$-invariant ($p(\Y | \X, g^{-1} \cdot \w) = p(\Y | \X, \w)$). Therefore, from Equation \eqref{eq:bayes} we get:
\begin{proposition}\label{thm:invariance_of_posterior}
$p(\w \mid \data)$ is $G$-invariant.
\end{proposition}
This invariance property causes mode multiplicity in the posterior. If $\w^*$ is a mode of $p(\w \mid \data)$, so is $g \cdot \w^*$ for all $g \in G$. This means that, for every mode of the posterior, there are up to $\lvert G \rvert = d_1! \cdots d_{L-1}!$ \textit{equivalent} modes. The significant number of equivalent modes poses a challenge for standard VI techniques, as we will explore in the next section.

\looseness=-1
It should be noted that MLPs and other deep learning architectures exhibit other weight space symmetries that are not considered in this work, such as scaling transformations \cite{Badrinarayanan2015UnderstandingSI, phuong2020functional, Entezari2021TheRO} and data dependent symmetries.

\section{Effects of Equivalent Modes in the Posterior on Unimodal VI Approximations} \label{sec:effect_of_permutations}

\textbf{Mode-seeking behavior of the KL divergence.} Variational Inference often employs unimodal approximate posteriors, such as the common mean-field approximation with a Gaussian variational posterior. A commonly cited justification for this approach is the mode-seeking behavior of the KL divergence \cite{pml1Book}. When approximating a fixed distribution $p(\w)$ with a simpler variational distribution $q_\theta(\w)$ via KL minimization, we expect to see a different behavior depending on the direction of the KL loss.
\vspace{-2mm}
\begin{enumerate}
    \item \textbf{Forward KL.} Minimizing $\KL(p(\w) \mid\mid q_\theta(\w))$ forces $q_\theta(\w)$ to include all areas for which $p(\w)$ assigns non-zero mass. This behavior is called \textit{mode-covering}.
    
    \item \textbf{Reverse KL.} Minimizing $\KL(q_\theta(\w) \mid\mid p(\w))$ 
    forces $q_\theta(\w)$ to exclude all the areas for which $p(\w)$ is zero. This results in $q_\theta(\w)$ assigning mass to few parts of the space, near $p(\w)$'s modes. This behavior is called \textit{mode-seeking}.
\end{enumerate}
\vspace{-2mm}
As seen in Equation \eqref{eq:elbo}, variational inference over BNNs employs reverse KL minimization. Therefore, we might intuitively expect a unimodal variational posterior to converge to one of the equivalent modes in the true posterior. Our experiments reveal a potential pitfall of this intuition. We demonstrate that for certain configurations of modes in the target distribution, a unimodal approximation can become ``stuck'' between two modes, even when using reverse KL minimization. 

\looseness=-1
\vspace{-1mm}
We approximate the Gaussian mixture $p_\alpha(\x) = \frac{1}{2} \left(\normdist(\x; \mathbf{0}, \mathbf{I}) + \normdist(\x; \alpha \mathbf{u}, \mathbf{I}) \right)$, where $\mathbf{u} \in \mathbb{R}^d$ is a unit vector and $\alpha$ parameterizes the distance between $p_\alpha(\x)$'s mixture components. We use a unimodal Gaussian variational approximation $q_{\m, \s}(\x) = \normdist(\x; \m, \s)$, with variational mean $\m \in \mathbb{R}^d$ and covariance matrix $\s \in \mathbb{R}^{d \times d}$. We initialize optimization with $(\m_0, \s_0) = (\mathbf{0}, \mathbf{I})$, corresponding to $p_\alpha(\x)$'s mixture component, in order to further encourage mode-seeking behavior. See additional details in Appendix \ref{apx:mode_interpolation}. The mode of the true optimum lies on the line segment $[\mathbf{0}, \alpha \mathbf{u}]$. Therefore, denoting the result of optimization by $(\m^*, \s^*)$, we interpret $\alpha^{-1} \lVert \m^* \rVert_2$ as a measure of the solution's \textit{interpolation} between $p_\alpha(\x)$'s components. $\alpha^{-1} \lVert \m^* \rVert_2 = \frac{1}{2}$ corresponds to perfect mid-point interpolation and $\alpha^{-1} \lVert \m^* \rVert_2 = 0$ means no interpolation.

Our results, summarized in Figure \ref{fig:revrese_kl}, suggest the existence of a threshold $\alpha^*$ (in this experiment $\alpha^* \approx 5$), controlling mode-seeking behavior. For mode discrepancy $\alpha \leq \alpha^*$, the optimization results are in fact \textit{mean-seeking} in the sense that $\m^*$ is a mid-point interpolation between between $0$ and $\alpha \textbf{u}$. Only for $\alpha \geq \alpha^*$, the behavior becomes mode-seeking. We empirically find that the threshold for mode-seeking behavior scales with the standard deviation of $p_\alpha(\x)$'s mixture components. For details, see Appendix \ref{apx:mode_interpolation}.

\begin{figure}[t]
    \centering
    \begin{tabular}{cc}
        \includegraphics[width=0.45\linewidth, trim={5 5 5 5},clip]{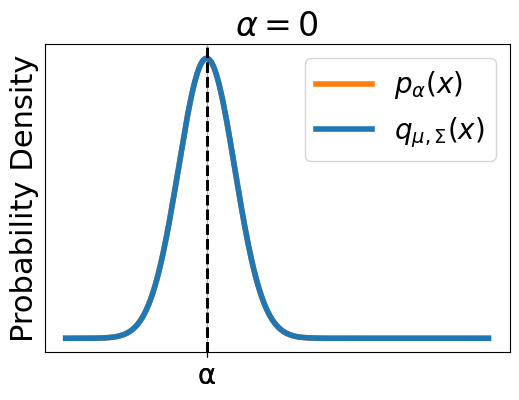} &
        \includegraphics[width=0.45\linewidth, trim={5 5 5 5},clip]{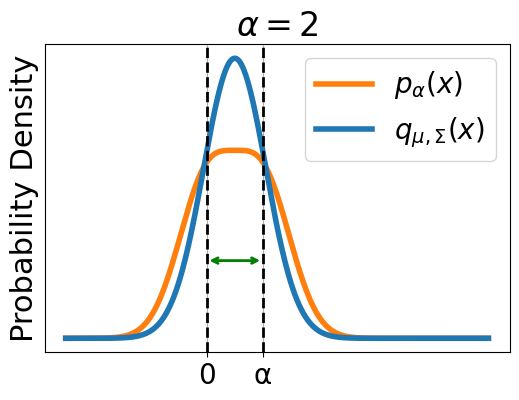} \\
        \includegraphics[width=0.45\linewidth, trim={5 5 5 5},clip]{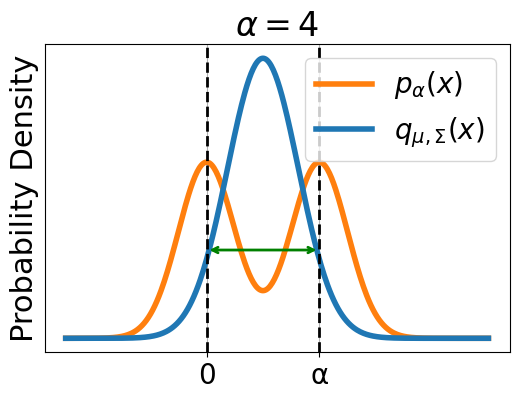} &
        \includegraphics[width=0.45\linewidth, trim={5 5 5 5},clip]{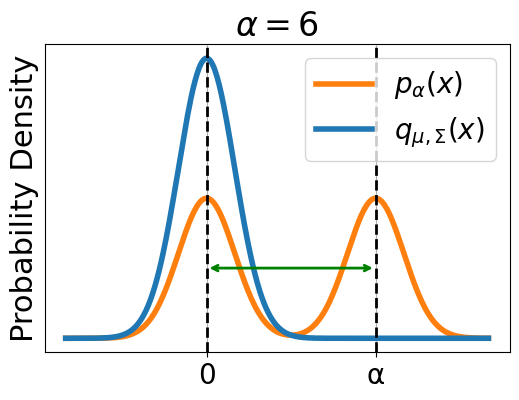} 
    \end{tabular}
    \caption{\textbf{Unimodal reverse KL minimizer of a bimodal target distribution.} Plots of the target distribution \textcolor{BurntOrange}{$q_\alpha(x)$} and the unimodal reverse KL minimizer \textcolor{RoyalBlue}{$q_{\m, \s}(x)$} in one dimension.}
    \label{fig:1d_example}
\end{figure}

\begin{figure*}[ht!]
    \centering
    \begin{tabular}{ccc}
        \includegraphics[width=0.3\textwidth, trim={5 0 0 0},clip]{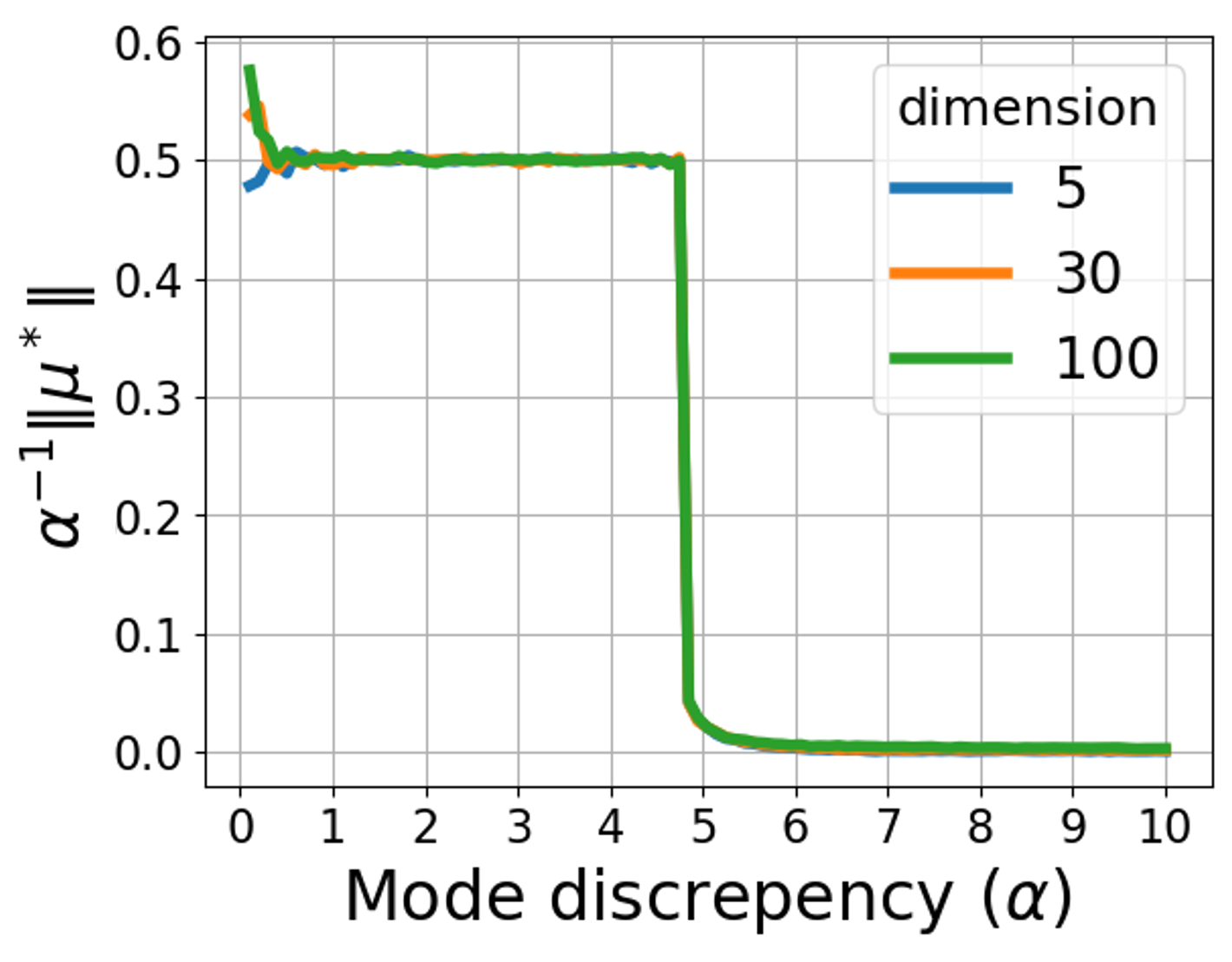} &
        \includegraphics[width=0.3\textwidth, trim={5 0 0 0},clip]{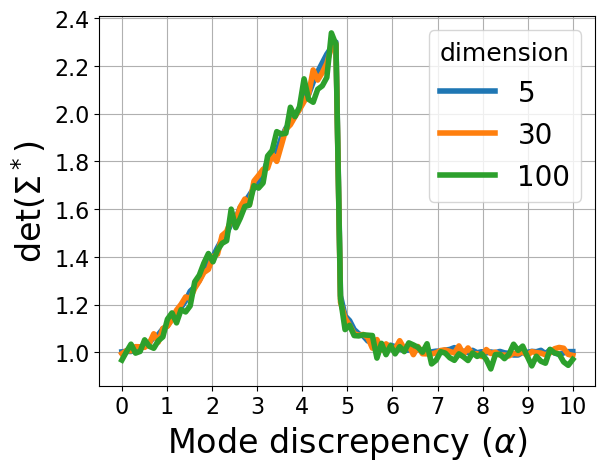} &
        \includegraphics[width=0.3\textwidth, trim={5 0 0 0},clip]{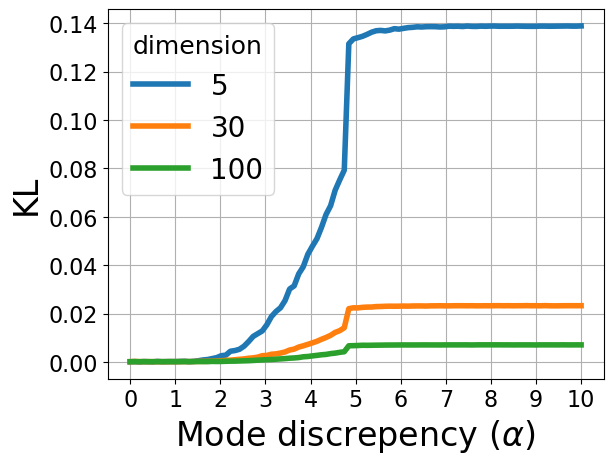} 
    \end{tabular}
    \caption{\textbf{behavior of the unimodal reverse KL minimizer.} We optimize $(\m^*, \s^*)$ to minimize $\KL(q_{\m, \s}(\x) \mid \mid p_\alpha(\x))$, where $q_{\m, \s}(\x) = \normdist(\x; \m, \s)$ is a Gaussian distribution and $p_\alpha(\x)$ is a mixture of two standard Gaussians with one component centered at $\mathbf{0}$ and another centered $\alpha$ away from $\mathbf{0}$. We plot the results of the optimization as a function of $\alpha$. From left to right, we plot: $\alpha^{-1}\lVert \m^* \rVert_2$ (quantifying the amount of interpolation between $p_\alpha(\x)$'s components), $\det(\s^*)$, and the optimal $\hat{\KL}$ reached by optimization.}
    \label{fig:revrese_kl}
\end{figure*}

\textbf{Interpolating equivalent modes of the posterior.} What are the effects of mode interpolation in the context of VI for BNNs? While interpolating posterior modes that arise from \textit{different} functions fitting the data might be desirable in some scenarios, in the case of \textit{equivalent} modes (i.e. ones that represent the same function) interpolation becomes detrimental. Interpolating equivalent modes causes the approximate posterior to assign a higher likelihood to weights that lie between multiple equivalent modes. These weights may correspond to less likely functions compared to those that are closer to a single mode. This means that sampling from the variational posterior results in less likely functions, underfitting the data and degrading predictive performance. This intuition is verified empirically in Section \ref{sec:experiments}. Additionally, as demonstrated in Figure \ref{fig:revrese_kl}, interpolating solutions have higher variance than each of the mixture components, suggesting uncertainty may be overestimated.

\textbf{Equivalent mode proximity.} The above discussion suggests that unimodal approximate distributions would tend to interpolate closely spaced equivalent modes in the posterior, which can lead to underfitting. This raises the question: \textit{are equivalent modes in the posterior close to one another?} The following theorem answers this question affirmatively for wide, single hidden layer MLPs.
\vspace{1mm}
\begin{theorem}\label{thm:mode_proximity}
    Given an MLP with a single hidden layer of width $d_h$, input dimension $d_i$ and output dimension $d_o$, for every weight configuration $\w \in \W$ there exists a non-trivial weight space permutation $g \in G$ such that
    \begin{equation*}
        \lVert \w - g \cdot \w \rVert_2 = \mathcal{O} \left(\frac{d_i + d_o}{\log(d_h)}\lVert \w \rVert_2 \right)
    \end{equation*}
\end{theorem}
See proof in Appendix \ref{apx:mode_proximity}. As a result, in wide neural networks (e.g. fixing $d_i, d_o$ and taking $d_h \to \infty$), the distance between a mode $\w^*$ and the closest equivalent mode, relative to $\lVert \w^* \rVert_2$, goes to $0$. This suggests higher relative mode proximity in wide networks.
\vspace{-2mm}

\section{Symmetrization of the Variational Posterior} \label{sec:symmetrization}
\looseness=-1
To address the challenge of approximating a permutation invariant posterior with a large number of equivalent modes, we propose using a $G$-invariant variational posterior. There are several approaches for designing invariant models, including the use of \textit{invariant architectures} \cite{Cohen2016GroupEC, deepsets, Hartford2018DeepMO, Maron2018InvariantAE, DWS} and transforming the input to a \textit{canonical form} \cite{kaba2022equivariance, Tai2019EquivariantTN, kofinas2021rototranslated}. In this work, we opt for a simple, architecture-agnostic strategy: averaging a non-invariant model over the orbits of the symmetry group \cite{Puny2021FrameAF, Yarotsky2018UniversalAO, Benton2020LearningII}. We call this process \textit{$G$-symmetrization}. Concretely, given an approximate posterior family $q_\theta(\w)$, its $G$-symmetrization is
\begin{equation}\label{eq:symmetrization}
    q_\theta^G (\w) = \frac{1}{|G|} \sum_{g \in G} q_\theta(g^{-1} \cdot \w)
\end{equation}
\looseness=-1
See illustration in Figure \ref{fig:symmterization_of_the_posterior}. $q_\theta^G(\w)$ can be viewed as a coupling of two random variables, $\g \in G$ and $\w \in \W$, with 
\begin{equation*}
q_\theta^G(\g, \w) = q_\theta^G(\w \mid \g) \cdot q_\theta^G(\g) = q_\theta(\g^{-1} \cdot \w) \cdot \frac{1}{|G|}
\end{equation*}
\looseness=-1
and $q_\theta^G(\w) = \E_{g \sim q_\theta^G(g)} (q_\theta^G(\w \mid g))$. Additionally, since the representation of $G$ is orthonormal, $g_\# q_\theta(\w) = q_\theta(g^{-1} \cdot \w)$ (see Section \ref{sec:perlimiaries}). Thus, Equation \eqref{eq:symmetrization} can be written as $q_\theta^G(\w) = \E_{\g \sim \text{uniform}(G)} \left( \g_\# q_\theta (\w)\right)$. For a Gaussian variational posterior $q_{\m, \s}(\w) = \normdist(\w ; \m, \s)$ the symmetrization takes the form of a symmetric Gaussian mixture $q_{\m, \s}^G(\w) = \frac{1}{|G|} \sum_{g \in G} \normdist(\w ; g \cdot \m, g \cdot \s)$, where $g \cdot \s := \rho(g)^\top \s \rho(g)$. In the case of mean-field VI, the covariance matrix is diagonal $\s = \diag(\boldsymbol{\sigma})$, so this becomes $g \cdot \diag(\boldsymbol{\sigma}) = \diag(g \cdot \boldsymbol{\sigma})$.

\looseness=-1
Unless $q_\theta(\w)$ is itself $G$-invariant (in which case $q_\theta^G \equiv q_\theta$), its symmetrization $q_\theta^G(\w)$ has a \textit{strictly better} fit to the posterior (see Corollary \ref{cor:posterior_fit}). However, sampling a model from  $q_\theta^G(\w)$ (i.e. $\f^{\w}(\x)$, $\w \sim q_\theta^G(\w)$) is equivalent to sampling a model from $q_\theta(\w)$. This is because drawing a sample from $q_\theta^G(\w)$ is equivalent to sampling $\w \sim q_\theta(\w)$ and transforming it by a random $g \sim \text{uniform}(G)$. Since $\f^{\w} \equiv \f^{g \cdot \w}$, the resulting function space distributions are the same (for a more elborate discussion, see appendix \ref{apx:sampling}). Therefore, the only difference in performing VI using $q_\theta^G(\w)$ lies in training for the optimal variational parameters $\theta$ by maximizing the ELBO objective associated with $q_\theta^G(\w)$. We denote $q_\theta^G(\w)$'s ELBO by $\loss_\VI^G(\theta)$ and refer to it as the \textit{symmetrized ELBO}. Importantly, even though we train via $\loss_\VI^G(\theta)$, we can still use the original distribution $q_\theta(\w)$ for making predictions.

\textbf{Structure of the symmetrized ELBO.} While the vast number of weight space permutations induced by the symmetry group $G$ might initially suggest that training using the fully symmetrized ELBO $\loss_\VI^G(\theta)$ is intractable, the inherent symmetry of $q_\theta^G(\w)$ offers a key advantage and simplifies the structure of  $\loss_\VI^G(\theta)$.

\begin{theorem}
\label{thm:elbo_correction}
    Given a variational posterior $q_\theta(\w)$
    \begin{equation}\label{eq:elbo_correction}
    \begin{split}
        \loss_\VI^G(\theta) &= \loss_\VI(\theta) + H(q_\theta^G(\w)) - H(q_\theta(\w)) \\
        &= \loss_\VI(\theta) + I(g; \w)
    \end{split}
    \end{equation}
    where $H(q_\theta(\w))$ is the entropy of $q_\theta(\w)$, $H(q_\theta^G(\w))$ is the entropy of $q_\theta^G(\w)$, and $I(g; \w)$ is the mutual information of $g$ and $\w$ under $q_\theta^G(g, \w)$.  
\end{theorem}
See proof in Appendix \ref{apx:elbo_correction}. Theorem \ref{thm:elbo_correction} serves two purposes. In the next section, we use Equation \eqref{eq:elbo_correction} to evaluate $\loss_\VI^G(\theta)$ in practice by estimating the mutual information $I(g; \w)$. In addition, Theorem \ref{thm:elbo_correction} offers theoretical insight, detailed below, on the posterior fit of the symmetrized variational approximation.

\vspace{1mm}
\begin{corollary}\label{cor:posterior_fit}
    \looseness=-1
    For all variational parameter configurations $\theta$, $\loss_\VI^G(\theta) \geq \loss_\VI(\theta)$, with equality iff $q_\theta(\w)$ is $G$-invariant.
\end{corollary}
\begin{proof}[Proof of Corollary \ref{cor:posterior_fit}]
    Since mutual information is non-negative, 
    \begin{equation}\label{eq:ineq}
        \loss_\VI^G(\theta) = \loss_\VI(\theta) + I(g; \w) \geq \loss_\VI(\theta)
    \end{equation}
    Equality in \eqref{eq:ineq} means that $I(g; \w) = 0$, and therefore that $g$ and $\w$ are independent under $q_\theta^G(g, \w)$. This, in turn, implies that $\forall g \in G, \, q_\theta^G(\w \mid g) = q_\theta^G(\w)$, and in particular
    \begin{equation}\label{eq:group_invariance}
        \forall g \in G, \, q_\theta^G(\w \mid g) = q_\theta^G(\w) = q_\theta^G(\w \mid e)
    \end{equation}
    where $e \in G$ is the identity element. By definition, $q_\theta^G(\w \mid g) = q_\theta(g^{-1} \cdot \w)$, so \eqref{eq:group_invariance} becomes $\forall g \in G, \, q_\theta(g^{-1} \cdot \w) = q_\theta(e^{-1} \cdot \w) = q_\theta(\w)$, which means $q_\theta(\w)$ is $G$-invariant. In the other direction, if $q_\theta(\w)$ is $G$-invariant, its symmetrization is trivial $q_\theta^G \equiv q_\theta$, thus $\loss_\VI^G(\theta) = \loss_\VI(\theta)$.
\end{proof}
In other words, unless $q_\theta(\w)$ is $G$-invariant to begin with, there is a strict gap between $\loss_\VI^G(\theta)$ and $\loss_\VI(\theta)$, implying $q_\theta^G(\w)$ has a strictly better fit to the true posterior, as mentioned previously.

\looseness=-1
Note that given optimums $\theta^* \in \text{argmax}_{\theta} \loss_\VI(\theta)$ and $\theta_G^* \in \text{argmax}_{\theta} \loss_\VI^G(\theta)$ we might have $\loss_\VI(\theta^*) \geq \loss_\VI(\theta_G^*)$, but we'll always have $\loss_\VI^G(\theta_G^*) \geq \loss_\VI^G(\theta^*) > \loss_\VI(\theta^*)$. Thus, symmetrized VI always results in a better posterior, further motivating the optimization of $\loss_\VI^G(\theta)$. To tractably do so, Theorem \ref{thm:elbo_correction} suggests focusing on the mutual information term $I(g; \w) = H(q_\theta^G(\w)) - H(q_\theta(\w))$.

\vspace{1mm}
\looseness=-1
\textbf{Estimating the symmetrized ELBO.} As is common in VI, we assume both $\loss_{\VI}(\theta)$ and $H(q_\theta(\w))$ are tractable to evaluate (e.g. in the case of Gaussian prior and variational posterior). Therefore, in order to evaluate $\loss_\VI^G(\theta)$, it is enough to evaluate the entropy of the symmetrized variational posterior $H(q_\theta^G(\w))$. The entropy of mixture models typically lacks a closed-form expression, which means we need to resort to an estimator. However, standard estimators for mixture model entropy \cite{huber, Kolchinsky2017EstimatingME} are not applicable in this case due to the large number of mixture components in $q_\theta^G(\w)$. Fortunately, we can leverage the symmetric structure of $q_\theta^G(\w)$ to construct an efficient estimator for the entire mutual information term, and $H(q_\theta^G(\w))$ in particular.

\vspace{2mm}
\looseness=-1
Mutual information estimation is core to many machine learning tasks such as representation learning \cite{Oord2018RepresentationLW, Wu2020OnMI}, reinforcement learning \cite{Nachum2018NearOptimalRL, Shu2020PredictiveCF}, and more. In order to estimate $I(g; \w)$, we leverage the InfoNCE bound \cite{Oord2018RepresentationLW, Poole2019OnVB}, which is extensively used in contrastive learning, and has well-understood bias and variance \cite{Song2020Understanding}. Using InfoNCE, and leveraging the symmetric structure of $q_\theta^G(\w)$ and the identity $H(q_\theta^G(\w)) - H(q_\theta(\w)) = I(\g; \w)$, we obtain the following estimators for $H(q_\theta^G(\w))$.
\begin{theorem}\label{thm:entropy_bound}
    Let $q_\theta(\w)$ be a variational posterior. The estimator
    \thinmuskip=2mu \medmuskip=2mu
    \begin{equation*}
        H^K(\theta) = \E \left( -\log\left(\frac{1}{K}\left(q_\theta(\w) + \sum_{j = 1}^{K-1} q_\theta(\g_j^{-1} \cdot \w) \right)\right)\right)
    \end{equation*}
    where the expectation is over $\w \sim q_\theta(\w)$ and $\g_1, \dots, g_{K - 1} \sim \text{uniform}(G)$ i.i.d. satisfies
    \begin{equation*}
        H^K(\theta) \leq H(q_\theta^G(\w)), \, \lim_{K \to \infty} H^K(\theta) = H(q_\theta^G(\w))
    \end{equation*}
\end{theorem}
\vspace{2mm}
See proof in Appendix \ref{apx:mi_estimation}. The symmetric structure of $q_\theta^G(\w)$ allows us to simplify the original InfoNCE estimator, reducing its computational complexity from $\mathcal{O}(K^2)$ to $\mathcal{O}(K)$. Since $H^K(\theta)$ becomes unbiased as $K \to \infty$, we can trade-off the number of samples, $K$, and the bias in our estimator. We therefore estimate $\loss_\VI^G(\theta)$ as
\begin{equation*}
    \loss_\VI^K(\theta) = \loss_\VI(\theta) - H(q_\theta(\w)) + H^K(\theta)
\end{equation*}
Combining Theorems \ref{thm:elbo_correction} and \ref{thm:entropy_bound} we get $\loss_\VI^K(\theta) \leq \loss_\VI^G(\theta)$ and $\lim_{K \to \infty} \loss_\VI^K(\theta) = \loss_\VI^G(\theta)$.

\vspace{2mm}
\textbf{Optimizing the symmetrized ELBO.} We use Monte Carlo estimation to approximate both the expected log-likelihood term and the entropy estimator $H^K(\theta)$ within $\loss_\VI^K(\theta)$. We assume closed-form expressions for $\KL(q_\theta(\w) \mid \mid p(\w))$ and $H(q_\theta(\w))$, which is common in many VI setups. Given a mini-batch of $M$ data points $\mathcal{B} = \left(\X_{\mathcal{B}}, \Y_{\mathcal{B}}\right)$ drawn from a dataset of size $N$, we sample $S$ weight configurations $\w_1, \dots, \w_{S} \sim q_\theta(\w)$ and use them to estimate $\hat{\loss}_\VI(\theta)$ as
\begin{equation*}
     \frac{N}{M} \frac{1}{S}\sum_{i = 1}^{S} \log(p(\Y_{\mathcal{B}} \mid \X_{\mathcal{B}}, \w_i)) + \KL(q_\theta(\w) \mid \mid p(\w))
\end{equation*}
To estimate $H^K(\theta)$, we proceed by sampling $K -1$ random weight space permutations $g_1, \dots, g_{K-1} \sim \text{uniform}(G)$, computing $\hat{H}^K(\theta)$ as
\begin{equation*}
     -\frac{1}{S}\sum_{i = 1}^S \log\left(\frac{1}{K}\left(q_\theta(\w_i) + \sum_{j = 1}^{K-1} q_\theta(g_j^{-1} \cdot \w_i) \right)\right)
\end{equation*}
Combining both estimators, we obtain the final estimator for the symmetrized ELBO objective:
\begin{equation*}
    \hat{\loss}_\VI^K(\theta) = \hat{\loss}_\VI(\theta) -H(q_\theta(\w)) +\hat{H}^K(\theta)
\end{equation*}
We optimize $\hat{\loss}_\VI^K(\theta)$ via gradient ascent, backpropagating gradients through the Monte Carlo estimators using the reparametrization trick \cite{Kingma2013StochasticGV}.

\begin{figure}[t]
\centering
  \begin{subfigure}[t]{\linewidth}
    \centering
    \caption{Ground truth posterior}
    \begin{tabular}{cccc}
      \includegraphics[width=0.2\textwidth, trim={100 30 100 0},clip]{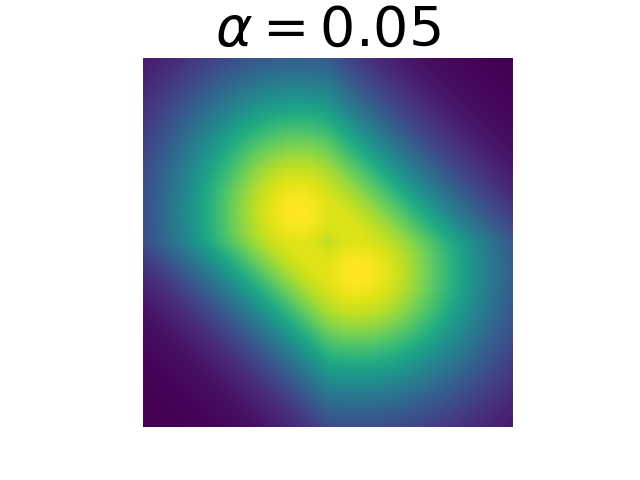} &  
      \includegraphics[width=0.2\textwidth, trim={100 30 100 0},clip]{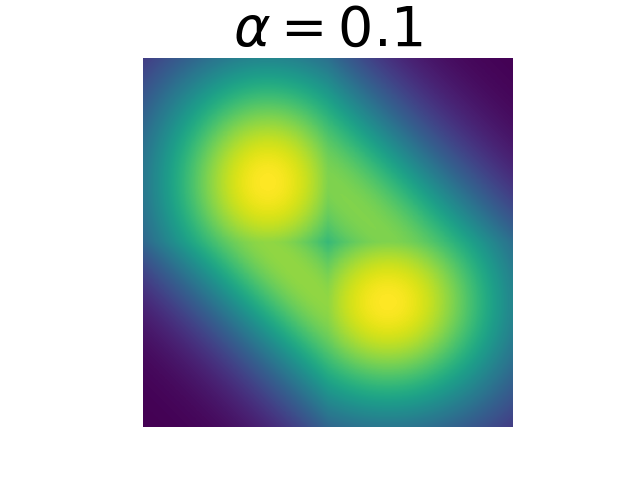} &  
      \includegraphics[width=0.2\textwidth, trim={100 30 100 0},clip]{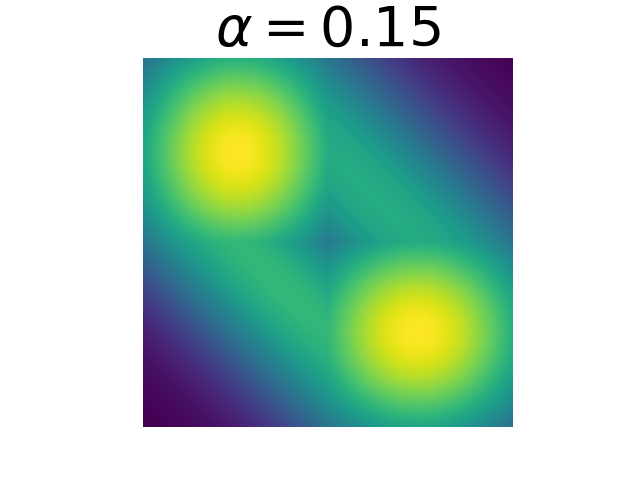} &
      \includegraphics[width=0.2\textwidth, trim={100 30 100 0},clip]{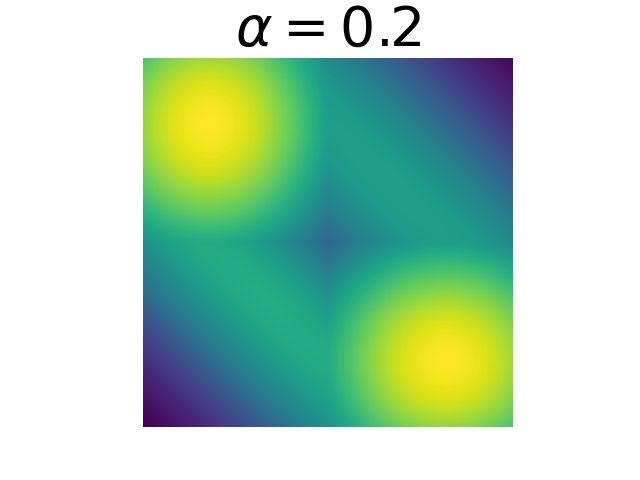}
    \end{tabular}
    \label{fig:exact_posterior}
  \end{subfigure}
  \hfill
  \begin{subfigure}[b]{\linewidth}
    \centering
    \caption{Mean-field VI with a Gaussian approximate posterior}
    \begin{tabular}{cccc}
      \includegraphics[width=0.2\linewidth, trim={100 30 100 0},clip]{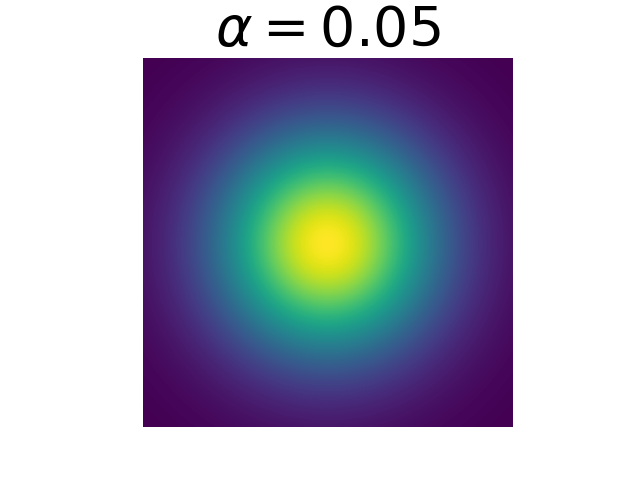} &  
      \includegraphics[width=0.2\linewidth, trim={100 30 100 0},clip]{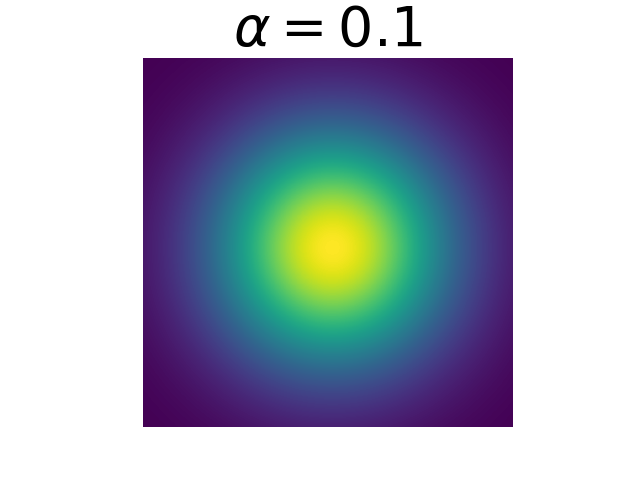} &  
      \includegraphics[width=0.2\linewidth, trim={100 30 100 0},clip]{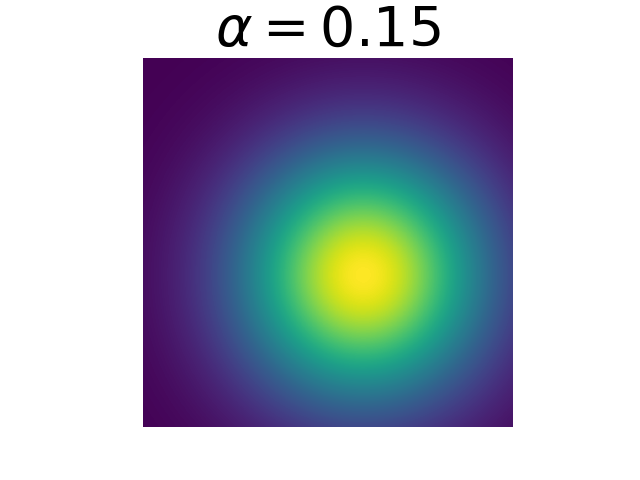} &
      \includegraphics[width=0.2\linewidth, trim={100 30 100 0},clip]{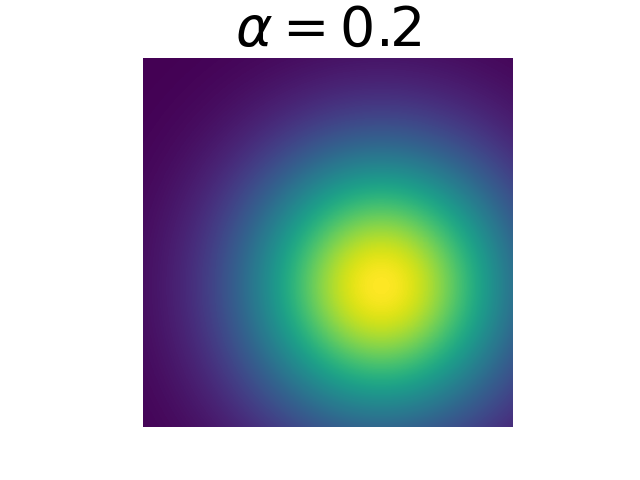}
    \end{tabular}
    \label{fig:mfg_approx}
  \end{subfigure}
  \begin{subfigure}[b]{\linewidth}
    \centering
    \caption{Symmetric Gaussian mixture approximate posterior}
    \begin{tabular}{cccc}
      \includegraphics[width=0.2\linewidth, trim={100 30 100 0},clip]{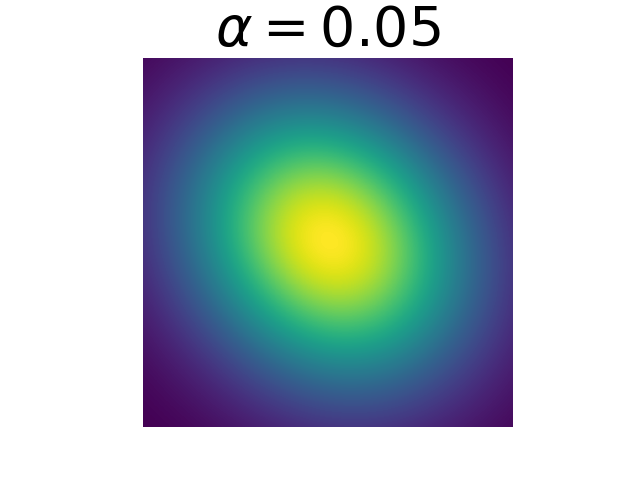} & 
      \includegraphics[width=0.2\linewidth, trim={100 30 100 0},clip]{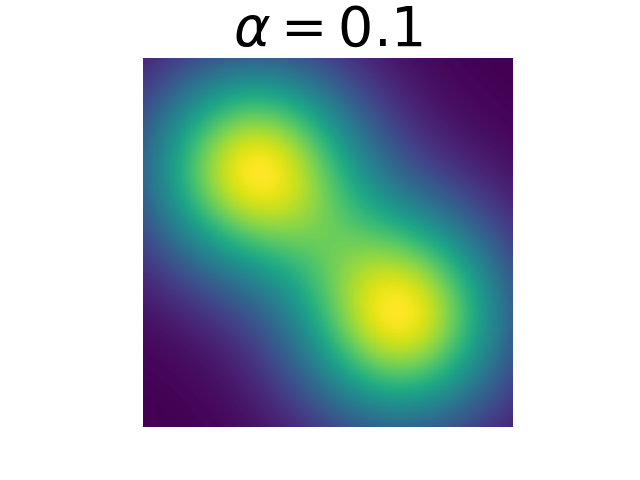} &
      \includegraphics[width=0.2\linewidth, trim={100 30 100 0},clip]{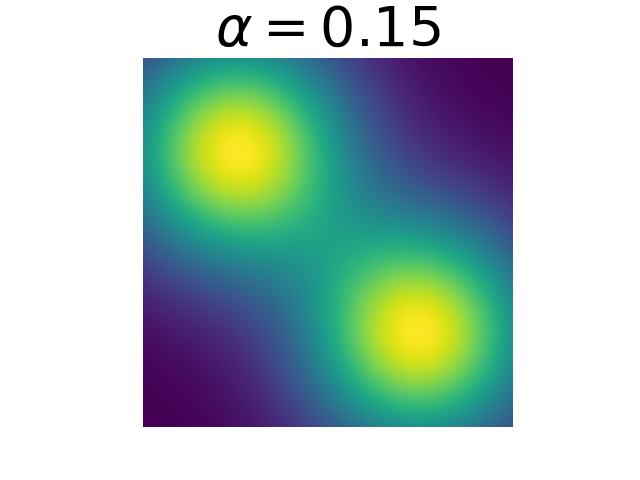} & 
      \includegraphics[width=0.2\linewidth, trim={100 30 100 0},clip]{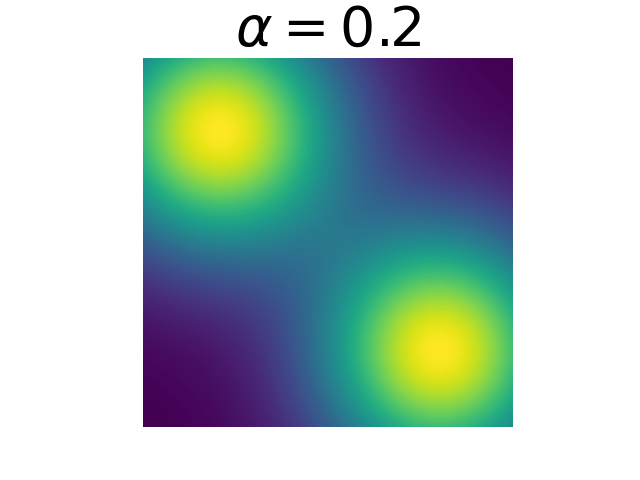}
    \end{tabular}
    \label{fig:sgm_approx}
  \end{subfigure}
\caption{\textbf{Posterior plots for a tractable BNN.} Ground truth and approximate posterior plots for the model $\f^{w_1, w_2}(x) = \text{ReLU}(w_1 x) + \text{ReLU}(w_2 x)$ given data generated by $f_\alpha(x) = \alpha \lvert x \rvert =\text{ReLU}(\alpha x) + \text{ReLU}(-\alpha x)$.}
\label{fig:posterior_comparison}
\end{figure}

\section{Experiments} \label{sec:experiments}
\looseness=-1
In this section, we empirically demonstrate both the effects of weight space permutations and the performance of our proposed mitigation. To do so, we conduct two experiments. First, we look at a simple BNN, for which exact posterior computation is tractable. We compare its ground truth posterior to VI approximations given by both mean-field VI and our proposed symmetrization. Secondly, we train MLPs to classify MNIST digits \cite{MNIST} using both inference techniques. We show that our method achieves better predictive performance, and that the gap grows with model width, validating our intuition from Section \ref{sec:effect_of_permutations}.

\looseness=-1
In both experiments, we use mean-field VI with a Gaussian approximate posterior (referred to as MFVI) as the base variational posterior. That is, $q_\theta(\w) = q_{\m, \boldsymbol{\sigma}}(\w) = \normdist(\w; \m, \diag(\boldsymbol{\sigma}))$, with $\m, \boldsymbol{\sigma} \in \W$. We train $q_{\m, \boldsymbol{\sigma}}(\w)$ using both $\hat{\loss}_\VI(\theta)$ and $\hat{\loss}_\VI^K(\theta)$. The latter corresponds to training a symmetric Gaussian mixture variational posterior (referred to as SGM) of the form $q_{\m, \boldsymbol{\sigma}}(\w) = \frac{1}{|G|}\sum_{g \in G} \normdist(\w; g \cdot \m, \diag(g \cdot \boldsymbol{\sigma}))$. We place an isotropic Gaussian prior on the weights and optimize using Adam \cite{Kingma2014AdamAM}.

\looseness=-1
At test time, we take the average prediction of models sampled from the variational posterior, i.e. $\hat{\f}(\x) = \frac{1}{S}\sum_{i = 1}^S\f^{\w_i}(\x)$, with $\w_1, \dots, \w_S \sim q_\theta(\w)$. We evaluate $\hat{\f}(\x)$ on the test set with different random seeds, and report average MSE or accuracy $\pm$ standard deviation. We additionally report ELBO values for the training set, which can be interpreted as measuring the fit of our approximation to the true posterior.

\begin{table*}
    \centering
        \resizebox{\textwidth}{!}{
            \begin{tabular}{ccccc}
                \toprule
                \multirow{2}{*}{\textbf{Method}} & \multicolumn{4}{c}{\textbf{Hidden dimension}} \\
                \cmidrule{2-5}
                & $\mathbf{5}$ & $\mathbf{10}$ & $\mathbf{20}$ & $\mathbf{30}$ \\
                \hline
                Mean-field VI with a Gaussian approximate posterior & $88.379 \pm 0.023 \%$ & $93.217 \pm 0.022 \%$ & $94.836 \pm 0.025 \%$ & $95.432 \pm 0.023 \%$ \\
                Symmetric Gaussian mixture VI ($K=5$) & $\mathbf{88.423} \pm 0.027 \%$ & $\mathbf{93.221} \pm 0.021 \%$ & $94.899 \pm 0.035 \%$ & $95.534 \pm 0.033 \%$ \\ 
                Symmetric Gaussian mixture VI ($K=10$) & $88.405\pm 0.034 \%$ &  $93.220\pm 0.023 \%$ &   $94.896\pm 0.026 \%$ &  $95.540\pm 0.025 \%$ \\ 
                Symmetric Gaussian mixture VI ($K=20$) &  $88.408 \pm 0.026 \%$ & $\mathbf{93.221} \pm 0.027 \%$ & $\mathbf{94.905} \pm 0.024 \%$ & $\mathbf{95.552} \pm 0.023 \%$ \\
                \bottomrule
            \end{tabular}
        }
    \caption{\textbf{Test accuracy on MNIST.} Predictions taken from an average of $1000$ models sampled from the approximate posterior. Results are shown as mean $\pm$ standard deviation over $10$ runs with different random seeds. Higher hidden dimension leads to an increase in the number of equivalent modes in the true posterior and thus a larger performance increase over mean-field VI.} \label{tab:mnist}
\end{table*}

\looseness=-1
\textbf{Tractable posterior neural networks.} Due to the intractability of ground truth posterior distributions for large models, our first experiment involves a small BNN for which we can analytically compute the posterior. We use a two-layer MLP with scalar input and output, two hidden neurons, ReLU activations, no biases, and a fixed output layer. This design yields a model with two trainable parameters $\w = (w_1, w_2)$ and a forward pass of $\f^{\w}(x) = \text{ReLU}(w_1 x) + \text{ReLU}(w_2 x)$. This is the smallest example that still exhibits weight space permutation symmetries, which in this case consist only of the identity permutation and the transposition $(w_1, w_2) \mapsto (w_2, w_1)$. This setup allows us to evaluate the posterior fit for both inference techniques with high accuracy.

\looseness=-1
We train on synthetic data generated by the function $f_\alpha(x) = \alpha \lvert x \rvert = \text{ReLU}(\alpha x) + \text{ReLU}(-\alpha x)$. The weight configurations $(w_1^*, w_2^*) = (\alpha, -\alpha)$ and $(w_1^*, w_2^*) = (-\alpha, \alpha)$ are equivalent optimal solutions that constitute the two equivalent modes of the true posterior, plotted in Figure \ref{fig:exact_posterior}. The distance between the modes grows linearly with $\alpha$. We use training and test sets of size $N=100$, generated as $\left((x_1, f_\alpha(x_1)), \dots, (x_N, \f_\alpha(x_N))\right)$ with $x_1, \dots, x_N \sim \text{uniform}\left([-10, 10]\right)$ i.i.d. We train our approximate inference methods using mini-batches of $10$ data points with $\text{learning rate} = 5 \times 10^{-3}$. Results are collected after $10$ epochs.

We plot the posterior approximations given by both MFVI and SGM (with $K = 2$) in Figure \ref{fig:posterior_comparison}. Figure \ref{fig:mfg_approx} demonstrates our intuition from Section \ref{sec:effect_of_permutations} and shows that the unimodal approximate posterior interpolates between the modes. This translates to poor predictive performance, as demonstrated in Table \ref{tab:abs_experiment_random_init}. On the other hand, Figure \ref{fig:sgm_approx} demonstrates the more accurate posterior fit of the symmetrization. This leads to higher (symmetrized) ELBO values, detailed in Table \ref{tab:abs_experiment_elbo_g}, and better predictive performance, detailed in Table \ref{tab:abs_experiment_random_init}.

\begin{table}
    \centering
        \resizebox{\columnwidth}{!}{
            \begin{tabular}{ccccc}
                \toprule
                \multirow{2}{*}{\textbf{Method}} & \multicolumn{4}{c}{$\boldsymbol{\alpha}$} \\
                \cmidrule{2-5}
                & $\mathbf{0.05}$ & $\mathbf{0.1}$ & $\mathbf{0.15}$ & $\mathbf{0.2}$ \\
                \midrule
                MFVI &  $0.11 \pm 0.22$ & $0.12 \pm 0.16$ & $0.13 \pm 0.15$ & $0.18 \pm 0.19$ \\
                SGM ($K = 2$) & $\mathbf{0.09} \pm 0.19$ & $\mathbf{0.09} \pm 0.17$ & $\mathbf{0.10} \pm 0.15$ & $\mathbf{0.12} \pm 0.16$ \\ 
                \bottomrule
            \end{tabular}
        }
    \caption{\textbf{Test MSE on synthetic dataset.} Predictions are taken from an average of $1000$ models sampled from the approximate posterior. Results are shown as mean $\pm$ standard deviation over $10$ runs with different random seeds.}
    \label{tab:abs_experiment_random_init}
\end{table}

\begin{table}
    \centering
        \resizebox{\columnwidth}{!}{
            \begin{tabular}{cccccc}
                \toprule
                \multirow{2}{*}{\textbf{Method}} & \multirow{2}{*}{\textbf{ELBO}} &\multicolumn{4}{c}{$\boldsymbol{\alpha}$} \\
                \cmidrule{3-6}
                & & $\mathbf{0.05}$ & $\mathbf{0.1}$ & $\mathbf{0.15}$ & $\mathbf{0.2}$ \\
                \midrule
                \multirow{2}{*}{MFVI} & $\loss_\VI$ & $-0.022$ & $-0.025$ & $-0.032$ & $-0.042$ \\
                & $\loss_\VI^G $& $-0.009$ & $-0.013$ & $-0.021$ & $-0.031$ \\ 
                \midrule
                \multirow{2}{*}{SGM ($K = 2$)} &  $\loss_\VI$ & $-0.023$ & $-0.029$ & $-0.035$ & $-0.043$ \\
                & $\loss_\VI^G $ &  $-0.004$ & $-0.007$ & $-0.012$ & $-0.019$ \\
                \bottomrule
            \end{tabular}
        }
    \caption{\textbf{Normalized ELBO values on synthetic dataset.} $\loss_\VI^G$ is approximated as $\hat{\loss}_\VI^K$ with $K=500$ and estimated using $1000$ weight samples from $q_\theta(\w)$. Higher values are better as they imply better posterior fit. Notice that for both inference methods, $\loss_\VI \leq \loss_\VI^G$. The symmetrized variational posterior has higher $\loss_\VI^G$ values and marginally lower $\loss_\VI$ values.}
    \label{tab:abs_experiment_elbo_g}
\end{table}

\looseness=-1
\textbf{MNIST classifiers.} For the second experiment, we train BNNs to classify MNIST digits using both inference methods. We use single hidden layer MLPs with increasing hidden dimension as the underlying models. We train for $10$ epochs with a batch size of $100$ and learning rate $=10^{-3}$. For the symmetric Gaussian mixture posterior, results are reported for $K=5, 10, 20$. Accuracy results are displayed in Table \ref{tab:mnist}. The improvement in accuracy achieved by the symmetrization increases with model width, verifying our intuition from Theorem \ref{thm:mode_proximity}.

\looseness=-1
We empirically noticed that the variance of $\hat{H}^K(\theta)$ grows with model size. This can be mitigated by increasing the value of $K$, which reduces the variance of $\hat{H}^K(\theta)$ and leads to more stable training. Ideally, in order to unlock the full potential of this approach for larger models, we want to be able to reduce the variance of $\hat{H}^K(\theta)$ without using large values of $K$, which increase the computational cost. This could be achieved in the future using variance reduction techniques, such as importance sampling or stratified sampling.

\vspace{2mm}
\section{Conclusion} \label{sec:limitations}
\looseness=-1
In this work, we demonstrate the effects of weight space permutation symmetries in MLPs on unimodal VI posterior approximations. We propose a symmetrization mechanism to mitigate these effects, prove that it improves the posterior fit of the variational distribution, and demonstrate its performance improvements over mean-field VI. This approach offers an architecture-agnostic way to construct trainable permutation invariant variational posteriors, and paves the way for incorporating other weight-space symmetries in VI.

\looseness=-1
\textbf{Limitations.} One limitation of our approach is that the variance of $\hat{H}^K(\theta)$ grows model size, which currently limits its effectiveness for larger architectures. This could be mitigated in the future using variance reduction techniques such as importance sampling. One possible mitigation, in the case of symmetric Gaussian mixtures, is importance sampling $g_1, \dots, g_{K-1}$ with a proposal distribution that places lower mass on permutations associated with far modes, e.g. $f(g_j) \propto \exp\left(-\lVert g_j \cdot \m - \m \rVert_2\right)$. This gives more weight to permutations that lead to modes that are closely spaced, which are the ones influencing the entropy of $q_\theta^G(\w)$ the most. Implementing this, or other variance reduction strategies, is left for future work.

\looseness=-1
\textbf{Future directions.} Our approach achieves invariance by averaging the density over group orbits. Exploring alternative avenues for building invariant variational posteriors is a worthwhile research direction. Promising approaches include: \begin{inparaenum}[(1)] \item \textit{Invariant architectures.} Several invariant architectures for neural network weight spaces have recently been proposed \cite{DWS, Zhou2024UniversalNF, lim2024graph, kofinas2024graph}. Designing expressive weight space permutation invariant posteriors using these architectures is an interesting future direction. \item \textit{Input canonicalization.} Using weight space alignment techniques such as \cite{ainsworth2023git, penaetal, navon2024equivariant} to transform weights into a canonical form, we can potentially perform inference on the quotient space $\W / G$. \end{inparaenum} Finally, extending our framework to other weight space symmetries is another interesting direction. Most of the analysis in this paper can be extended to general finite/compact symmetry groups of weight spaces.

\section{Acknowledgments}
Yoav Gelberg is supported by the the Engineering and Physical Sciences Research Council (EPSRC) Centre for Doctoral Training in Autonomous and Intelligent Machines and Systems (grant reference EP/S024050/1). Yarin Gal is supported by a Turing AI Fellowship financed by the UK government’s Office for Artificial Intelligence, through UK Research and Innovation (grant reference EP/V030302/1) and delivered by the Alan Turing Institute.

\bibliography{references}

\begin{thebibliography}{50}
\providecommand{\natexlab}[1]{#1}
\providecommand{\url}[1]{\texttt{#1}}
\expandafter\ifx\csname urlstyle\endcsname\relax
  \providecommand{\doi}[1]{doi: #1}\else
  \providecommand{\doi}{doi: \begingroup \urlstyle{rm}\Url}\fi

\bibitem[Ainsworth et~al.(2023)Ainsworth, Hayase, and Srinivasa]{ainsworth2023git}
Ainsworth, S., Hayase, J., and Srinivasa, S.
\newblock Git re-basin: Merging models modulo permutation symmetries.
\newblock In \emph{The Eleventh International Conference on Learning Representations}, 2023.
\newblock URL \url{https://openreview.net/forum?id=CQsmMYmlP5T}.

\bibitem[Badrinarayanan et~al.(2015)Badrinarayanan, Mishra, and Cipolla]{Badrinarayanan2015UnderstandingSI}
Badrinarayanan, V., Mishra, B., and Cipolla, R.
\newblock Understanding symmetries in deep networks.
\newblock \emph{ArXiv}, abs/1511.01029, 2015.
\newblock URL \url{https://api.semanticscholar.org/CorpusID:18720237}.

\bibitem[Benton et~al.(2020)Benton, Finzi, Izmailov, and Wilson]{Benton2020LearningII}
Benton, G.~W., Finzi, M., Izmailov, P., and Wilson, A.~G.
\newblock Learning invariances in neural networks.
\newblock \emph{ArXiv}, abs/2010.11882, 2020.
\newblock URL \url{https://api.semanticscholar.org/CorpusID:225039835}.

\bibitem[Blundell et~al.(2015)Blundell, Cornebise, Kavukcuoglu, and Wierstra]{blundell2015weight}
Blundell, C., Cornebise, J., Kavukcuoglu, K., and Wierstra, D.
\newblock Weight uncertainty in neural network.
\newblock In \emph{International conference on machine learning}, pp.\  1613--1622. PMLR, 2015.

\bibitem[Chen et~al.(2021)Chen, Gan, Li, Guo, Chen, Gao, Chung, Xu, Zeng, Lu, Li, Carin, and Tao]{Chen2021SimplerFS}
Chen, J., Gan, Z., Li, X., Guo, Q., Chen, L., Gao, S., Chung, T., Xu, Y., Zeng, B., Lu, W., Li, F., Carin, L., and Tao, C.
\newblock Simpler, faster, stronger: Breaking the log-k curse on contrastive learners with flatnce.
\newblock \emph{ArXiv}, abs/2107.01152, 2021.
\newblock URL \url{https://api.semanticscholar.org/CorpusID:235727265}.

\bibitem[Cohen \& Welling(2016)Cohen and Welling]{Cohen2016GroupEC}
Cohen, T. and Welling, M.
\newblock Group equivariant convolutional networks.
\newblock \emph{ArXiv}, abs/1602.07576, 2016.
\newblock URL \url{https://api.semanticscholar.org/CorpusID:609898}.

\bibitem[Dusenberry et~al.(2020)Dusenberry, Jerfel, Wen, Ma, Snoek, Heller, Lakshminarayanan, and Tran]{Dusenberry2020EfficientAS}
Dusenberry, M.~W., Jerfel, G., Wen, Y., Ma, Y.-A., Snoek, J., Heller, K.~A., Lakshminarayanan, B., and Tran, D.
\newblock Efficient and scalable bayesian neural nets with rank-1 factors.
\newblock \emph{ArXiv}, abs/2005.07186, 2020.
\newblock URL \url{https://api.semanticscholar.org/CorpusID:218628552}.

\bibitem[Entezari et~al.(2021)Entezari, Sedghi, Saukh, and Neyshabur]{Entezari2021TheRO}
Entezari, R., Sedghi, H., Saukh, O., and Neyshabur, B.
\newblock The role of permutation invariance in linear mode connectivity of neural networks.
\newblock \emph{ArXiv}, abs/2110.06296, 2021.
\newblock URL \url{https://api.semanticscholar.org/CorpusID:238743980}.

\bibitem[Fortuin et~al.(2022)Fortuin, Garriga-Alonso, Ober, Wenzel, Ratsch, Turner, van~der Wilk, and Aitchison]{fortuin2022bayesian}
Fortuin, V., Garriga-Alonso, A., Ober, S.~W., Wenzel, F., Ratsch, G., Turner, R.~E., van~der Wilk, M., and Aitchison, L.
\newblock Bayesian neural network priors revisited.
\newblock In \emph{International Conference on Learning Representations}, 2022.
\newblock URL \url{https://openreview.net/forum?id=xkjqJYqRJy}.

\bibitem[Ghosh et~al.(2018)Ghosh, Yao, and Doshi-Velez]{pmlr-v80-ghosh18a}
Ghosh, S., Yao, J., and Doshi-Velez, F.
\newblock Structured variational learning of {B}ayesian neural networks with horseshoe priors.
\newblock In Dy, J. and Krause, A. (eds.), \emph{Proceedings of the 35th International Conference on Machine Learning}, volume~80 of \emph{Proceedings of Machine Learning Research}, pp.\  1744--1753. PMLR, 10--15 Jul 2018.
\newblock URL \url{https://proceedings.mlr.press/v80/ghosh18a.html}.

\bibitem[Guerrero~Peña et~al.(2023)Guerrero~Peña, Medeiros, Dubail, Aminbeidokhti, Granger, and Pedersoli]{penaetal}
Guerrero~Peña, F.~A., Medeiros, H., Dubail, T., Aminbeidokhti, M., Granger, E., and Pedersoli, M.
\newblock Re-basin via implicit sinkhorn differentiation.
\newblock pp.\  20237--20246, 06 2023.
\newblock \doi{10.1109/CVPR52729.2023.01938}.

\bibitem[Hartford et~al.(2018)Hartford, Graham, Leyton-Brown, and Ravanbakhsh]{Hartford2018DeepMO}
Hartford, J.~S., Graham, D.~R., Leyton-Brown, K., and Ravanbakhsh, S.
\newblock Deep models of interactions across sets.
\newblock In \emph{International Conference on Machine Learning}, 2018.
\newblock URL \url{https://api.semanticscholar.org/CorpusID:3819852}.

\bibitem[Hecht-Nielsen(1990)]{HECHTNIELSEN1990129}
Hecht-Nielsen, R.
\newblock On the algebraic structure of feedforward network weight spaces.
\newblock In ECKMILLER, R. (ed.), \emph{Advanced Neural Computers}, pp.\  129--135. North-Holland, Amsterdam, 1990.
\newblock ISBN 978-0-444-88400-8.
\newblock \doi{https://doi.org/10.1016/B978-0-444-88400-8.50019-4}.
\newblock URL \url{https://www.sciencedirect.com/science/article/pii/B9780444884008500194}.

\bibitem[Hinton \& van Camp(1993)Hinton and van Camp]{10.1145/168304.168306}
Hinton, G.~E. and van Camp, D.
\newblock Keeping the neural networks simple by minimizing the description length of the weights.
\newblock In \emph{Proceedings of the Sixth Annual Conference on Computational Learning Theory}, COLT '93, pp.\  5–13, New York, NY, USA, 1993. Association for Computing Machinery.
\newblock ISBN 0897916115.
\newblock \doi{10.1145/168304.168306}.
\newblock URL \url{https://doi.org/10.1145/168304.168306}.

\bibitem[Huber et~al.(2008)Huber, Bailey, Durrant-Whyte, and Hanebeck]{huber}
Huber, M.~F., Bailey, T., Durrant-Whyte, H., and Hanebeck, U.~D.
\newblock On entropy approximation for gaussian mixture random vectors.
\newblock In \emph{2008 IEEE International Conference on Multisensor Fusion and Integration for Intelligent Systems}, pp.\  181--188, 2008.
\newblock \doi{10.1109/MFI.2008.4648062}.

\bibitem[Kaba et~al.(2022)Kaba, Mondal, Zhang, Bengio, and Ravanbakhsh]{kaba2022equivariance}
Kaba, S.-O., Mondal, A.~K., Zhang, Y., Bengio, Y., and Ravanbakhsh, S.
\newblock Equivariance with learned canonicalization functions.
\newblock In \emph{NeurIPS 2022 Workshop on Symmetry and Geometry in Neural Representations}, 2022.
\newblock URL \url{https://openreview.net/forum?id=pVD1k8ge25a}.

\bibitem[Kingma \& Ba(2014)Kingma and Ba]{Kingma2014AdamAM}
Kingma, D.~P. and Ba, J.
\newblock Adam: A method for stochastic optimization.
\newblock \emph{CoRR}, abs/1412.6980, 2014.
\newblock URL \url{https://api.semanticscholar.org/CorpusID:6628106}.

\bibitem[Kingma \& Welling(2013)Kingma and Welling]{Kingma2013StochasticGV}
Kingma, D.~P. and Welling, M.
\newblock Stochastic gradient vb and the variational auto-encoder.
\newblock 2013.
\newblock URL \url{https://api.semanticscholar.org/CorpusID:51841161}.

\bibitem[Kofinas et~al.(2021)Kofinas, Nagaraja, and Gavves]{kofinas2021rototranslated}
Kofinas, M., Nagaraja, N.~S., and Gavves, E.
\newblock Roto-translated local coordinate frames for interacting dynamical systems.
\newblock In Beygelzimer, A., Dauphin, Y., Liang, P., and Vaughan, J.~W. (eds.), \emph{Advances in Neural Information Processing Systems}, 2021.
\newblock URL \url{https://openreview.net/forum?id=c3RKZas9am}.

\bibitem[Kofinas et~al.(2024)Kofinas, Knyazev, Zhang, Chen, Burghouts, Gavves, Snoek, and Zhang]{kofinas2024graph}
Kofinas, M., Knyazev, B., Zhang, Y., Chen, Y., Burghouts, G.~J., Gavves, E., Snoek, C. G.~M., and Zhang, D.~W.
\newblock Graph neural networks for learning equivariant representations of neural networks.
\newblock In \emph{The Twelfth International Conference on Learning Representations}, 2024.
\newblock URL \url{https://openreview.net/forum?id=oO6FsMyDBt}.

\bibitem[Kolchinsky \& Tracey(2017)Kolchinsky and Tracey]{Kolchinsky2017EstimatingME}
Kolchinsky, A. and Tracey, B.~D.
\newblock Estimating mixture entropy with pairwise distances.
\newblock \emph{ArXiv}, abs/1706.02419, 2017.
\newblock URL \url{https://api.semanticscholar.org/CorpusID:21868289}.

\bibitem[Kurle et~al.(2021)Kurle, Januschowski, Gasthaus, and Wang]{Kurle2021}
Kurle, R., Januschowski, T., Gasthaus, J., and Wang, Y.~B.
\newblock On symmetries in variational bayesian neural nets.
\newblock In \emph{NeurIPS 2021 Workshop on Bayesian Deep Learning}, 2021.

\bibitem[Kurle et~al.(2022)Kurle, Herbrich, Januschowski, Wang, and Gasthaus]{kurle2022on}
Kurle, R., Herbrich, R., Januschowski, T., Wang, B., and Gasthaus, J.
\newblock On the detrimental effect of invariances in the likelihood for variational inference.
\newblock In Oh, A.~H., Agarwal, A., Belgrave, D., and Cho, K. (eds.), \emph{Advances in Neural Information Processing Systems}, 2022.
\newblock URL \url{https://openreview.net/forum?id=ft4xGJ8tIZH}.

\bibitem[Laurent et~al.(2024)Laurent, Aldea, and Franchi]{laurent2024a}
Laurent, O., Aldea, E., and Franchi, G.
\newblock A symmetry-aware exploration of bayesian neural network posteriors.
\newblock In \emph{The Twelfth International Conference on Learning Representations}, 2024.
\newblock URL \url{https://openreview.net/forum?id=FOSBQuXgAq}.

\bibitem[LeCun \& Cortes(2005)LeCun and Cortes]{MNIST}
LeCun, Y. and Cortes, C.
\newblock The mnist database of handwritten digits.
\newblock 2005.
\newblock URL \url{https://api.semanticscholar.org/CorpusID:60282629}.

\bibitem[Lim et~al.(2024)Lim, Maron, Law, Lorraine, and Lucas]{lim2024graph}
Lim, D., Maron, H., Law, M.~T., Lorraine, J., and Lucas, J.
\newblock Graph metanetworks for processing diverse neural architectures.
\newblock In \emph{The Twelfth International Conference on Learning Representations}, 2024.
\newblock URL \url{https://openreview.net/forum?id=ijK5hyxs0n}.

\bibitem[Maron et~al.(2018)Maron, Ben-Hamu, Shamir, and Lipman]{Maron2018InvariantAE}
Maron, H., Ben-Hamu, H., Shamir, N., and Lipman, Y.
\newblock Invariant and equivariant graph networks.
\newblock \emph{ArXiv}, abs/1812.09902, 2018.
\newblock URL \url{https://api.semanticscholar.org/CorpusID:56895597}.

\bibitem[Murphy(2022)]{pml1Book}
Murphy, K.~P.
\newblock \emph{Probabilistic Machine Learning: An introduction}.
\newblock MIT Press, 2022.
\newblock URL \url{probml.ai}.

\bibitem[Nachum et~al.(2018)Nachum, Gu, Lee, and Levine]{Nachum2018NearOptimalRL}
Nachum, O., Gu, S.~S., Lee, H., and Levine, S.
\newblock Near-optimal representation learning for hierarchical reinforcement learning.
\newblock \emph{ArXiv}, abs/1810.01257, 2018.
\newblock URL \url{https://api.semanticscholar.org/CorpusID:52909341}.

\bibitem[Navon et~al.(2023)Navon, Shamsian, Achituve, Fetaya, Chechik, and Maron]{DWS}
Navon, A., Shamsian, A., Achituve, I., Fetaya, E., Chechik, G., and Maron, H.
\newblock Equivariant architectures for learning in deep weight spaces.
\newblock In \emph{Proceedings of the 40th International Conference on Machine Learning}, ICML'23. JMLR.org, 2023.

\bibitem[Navon et~al.(2024)Navon, Shamsian, Fetaya, Chechik, Dym, and Maron]{navon2024equivariant}
Navon, A., Shamsian, A., Fetaya, E., Chechik, G., Dym, N., and Maron, H.
\newblock Equivariant deep weight space alignment, 2024.
\newblock URL \url{https://openreview.net/forum?id=iT1ttQXwOg}.

\bibitem[Neal(1995)]{Neal1995BayesianLF}
Neal, R.~M.
\newblock Bayesian learning for neural networks.
\newblock 1995.
\newblock URL \url{https://api.semanticscholar.org/CorpusID:60809283}.

\bibitem[Phuong \& Lampert(2020)Phuong and Lampert]{phuong2020functional}
Phuong, M. and Lampert, C.~H.
\newblock Functional vs. parametric equivalence of re{\{}lu{\}} networks.
\newblock In \emph{International Conference on Learning Representations}, 2020.
\newblock URL \url{https://openreview.net/forum?id=Bylx-TNKvH}.

\bibitem[Poole et~al.(2019)Poole, Ozair, van~den Oord, Alemi, and Tucker]{Poole2019OnVB}
Poole, B., Ozair, S., van~den Oord, A., Alemi, A.~A., and Tucker, G.
\newblock On variational bounds of mutual information.
\newblock \emph{ArXiv}, abs/1905.06922, 2019.
\newblock URL \url{https://api.semanticscholar.org/CorpusID:155100374}.

\bibitem[Pourzanjani et~al.(2017)Pourzanjani, Jiang, and Petzold]{Pourzanjani2017ImprovingTI}
Pourzanjani, A.~A., Jiang, R.~M., and Petzold, L.~R.
\newblock Improving the identiﬁability of neural networks for bayesian inference.
\newblock 2017.
\newblock URL \url{https://api.semanticscholar.org/CorpusID:46932278}.

\bibitem[Puny et~al.(2021)Puny, Atzmon, Ben-Hamu, Smith, Misra, Grover, and Lipman]{Puny2021FrameAF}
Puny, O., Atzmon, M., Ben-Hamu, H., Smith, E.~J., Misra, I., Grover, A., and Lipman, Y.
\newblock Frame averaging for invariant and equivariant network design.
\newblock \emph{ArXiv}, abs/2110.03336, 2021.
\newblock URL \url{https://api.semanticscholar.org/CorpusID:238419638}.

\bibitem[Rossi et~al.(2023)Rossi, Singh, and Hannagan]{Rossi2023OnPS}
Rossi, S., Singh, A., and Hannagan, T.
\newblock On permutation symmetries in bayesian neural network posteriors: a variational perspective.
\newblock \emph{ArXiv}, abs/2310.10171, 2023.
\newblock URL \url{https://api.semanticscholar.org/CorpusID:264146533}.

\bibitem[Rudner et~al.(2022)Rudner, Chen, Teh, and Gal]{rudner2022fsvi}
Rudner, T. G.~J., Chen, Z., Teh, Y.~W., and Gal, Y.
\newblock {T}ractable {F}unction-{S}pace {V}ariational {I}nference in {B}ayesian {N}eural {N}etworks.
\newblock In \emph{Advances in Neural Information Processing Systems 35}, 2022.

\bibitem[Shu et~al.(2020)Shu, Nguyen, Chow, Pham, Than, Ghavamzadeh, Ermon, and Bui]{Shu2020PredictiveCF}
Shu, R., Nguyen, T.~D., Chow, Y., Pham, T., Than, K., Ghavamzadeh, M., Ermon, S., and Bui, H.~H.
\newblock Predictive coding for locally-linear control.
\newblock In \emph{International Conference on Machine Learning}, 2020.
\newblock URL \url{https://api.semanticscholar.org/CorpusID:211678253}.

\bibitem[Simsek et~al.(2020)Simsek, Brea, Illing, and Gerstner]{simsek2020weightspace}
Simsek, B., Brea, J., Illing, B., and Gerstner, W.
\newblock Weight-space symmetry in neural network loss landscapes revisited, 2020.
\newblock URL \url{https://openreview.net/forum?id=rkxmPgrKwB}.

\bibitem[Song \& Ermon(2020)Song and Ermon]{Song2020Understanding}
Song, J. and Ermon, S.
\newblock Understanding the limitations of variational mutual information estimators.
\newblock In \emph{International Conference on Learning Representations}, 2020.
\newblock URL \url{https://openreview.net/forum?id=B1x62TNtDS}.

\bibitem[Sun et~al.(2019)Sun, Zhang, Shi, and Grosse]{sun2018functional}
Sun, S., Zhang, G., Shi, J., and Grosse, R.
\newblock {FUNCTIONAL} {VARIATIONAL} {BAYESIAN} {NEURAL} {NETWORKS}.
\newblock In \emph{International Conference on Learning Representations}, 2019.
\newblock URL \url{https://openreview.net/forum?id=rkxacs0qY7}.

\bibitem[Tai et~al.(2019)Tai, Bailis, and Valiant]{Tai2019EquivariantTN}
Tai, K.~S., Bailis, P.~D., and Valiant, G.
\newblock Equivariant transformer networks.
\newblock In \emph{International Conference on Machine Learning}, 2019.
\newblock URL \url{https://api.semanticscholar.org/CorpusID:59523684}.

\bibitem[Tishby et~al.(1989)Tishby, Levin, and Solla]{118274}
Tishby, Levin, and Solla.
\newblock Consistent inference of probabilities in layered networks: predictions and generalizations.
\newblock In \emph{International 1989 Joint Conference on Neural Networks}, pp.\  403--409 vol.2, 1989.
\newblock \doi{10.1109/IJCNN.1989.118274}.

\bibitem[van~den Oord et~al.(2018)van~den Oord, Li, and Vinyals]{Oord2018RepresentationLW}
van~den Oord, A., Li, Y., and Vinyals, O.
\newblock Representation learning with contrastive predictive coding.
\newblock \emph{ArXiv}, abs/1807.03748, 2018.
\newblock URL \url{https://api.semanticscholar.org/CorpusID:49670925}.

\bibitem[Wenzel et~al.(2020)Wenzel, Roth, Veeling, Swiatkowski, Tran, Mandt, Snoek, Salimans, Jenatton, and Nowozin]{Wenzel2020AppendixFH}
Wenzel, F., Roth, K., Veeling, B.~S., Swiatkowski, J.~B., Tran, L.~H., Mandt, S., Snoek, J., Salimans, T., Jenatton, R., and Nowozin, S.
\newblock Appendix for: How good is the bayes posterior in deep neural networks really?
\newblock 2020.
\newblock URL \url{https://api.semanticscholar.org/CorpusID:211044145}.

\bibitem[Wu et~al.(2020)Wu, Zhuang, Moss{\'e}, Yamins, and Goodman]{Wu2020OnMI}
Wu, M., Zhuang, C., Moss{\'e}, M., Yamins, D. L.~K., and Goodman, N.~D.
\newblock On mutual information in contrastive learning for visual representations.
\newblock \emph{ArXiv}, abs/2005.13149, 2020.
\newblock URL \url{https://api.semanticscholar.org/CorpusID:218900934}.

\bibitem[Yarotsky(2018)]{Yarotsky2018UniversalAO}
Yarotsky, D.
\newblock Universal approximations of invariant maps by neural networks.
\newblock \emph{Constructive Approximation}, 55:\penalty0 407--474, 2018.
\newblock URL \url{https://api.semanticscholar.org/CorpusID:13745401}.

\bibitem[Zaheer et~al.(2017)Zaheer, Kottur, Ravanbakhsh, Poczos, Salakhutdinov, and Smola]{deepsets}
Zaheer, M., Kottur, S., Ravanbakhsh, S., Poczos, B., Salakhutdinov, R.~R., and Smola, A.~J.
\newblock Deep sets.
\newblock In Guyon, I., Luxburg, U.~V., Bengio, S., Wallach, H., Fergus, R., Vishwanathan, S., and Garnett, R. (eds.), \emph{Advances in Neural Information Processing Systems}, volume~30. Curran Associates, Inc., 2017.
\newblock URL \url{https://proceedings.neurips.cc/paper_files/paper/2017/file/f22e4747da1aa27e363d86d40ff442fe-Paper.pdf}.

\bibitem[Zhou et~al.(2024)Zhou, Finn, and Harrison]{Zhou2024UniversalNF}
Zhou, A., Finn, C., and Harrison, J.
\newblock Universal neural functionals.
\newblock \emph{ArXiv}, abs/2402.05232, 2024.
\newblock URL \url{https://api.semanticscholar.org/CorpusID:267547546}.

\end{thebibliography}
\bibliographystyle{icml2024}

\newpage
\appendix
\onecolumn
\section{Mode Interpolation: Experimental Details and Additional Results} \label{apx:mode_interpolation}

\paragraph{Experimental details.} Since $\KL(q_{\m, \s}(\x) \mid \mid p_\alpha(\x))$ does not have a closed-form expression, we minimize the empirical reverse KL which is estimated by sampling $S$ points $\x_1, \dots, \x_S$ from $q_{\m, \s}(\x)$ and computing:
\begin{equation}
    \hat{\KL} = \frac{1}{S} \sum_{i = 1}^S \log \left( \frac{q_{\m, \s}(\x_i)}{p_\alpha(\x_i)} \right)
\end{equation}
Using the reparametrization trick to compute gradients w.r.t $\m$ and $\s$, we minimize $\hat{\KL}$ via stochastic gradient descent. We use $S = 5 \times 10^3$ samples to approximate $\hat{\KL}$ and perform $3 \times 10^3$ training steps with learning rate $= 10^{-2}$, empirically observing convergence of the loss.

\paragraph{Scaling of the threshold for mode-seeking behavior.} As mentioned in the main text, the threshold for mean-seeking vs. mode-seeking behavior in the experimental setup described in Section \ref{sec:effect_of_permutations} scales with the standard deviation of $p_\alpha(\x)$'s mixture components. To see that, we define $p_\alpha^\sigma(\x) = \frac{1}{2}\left(\normdist(\x; \mathbf{0}, \sigma^2 \mathbf{I}) + \normdist(\x; \alpha \mathbf{u}, \sigma^2 \mathbf{I})\right)$, and repeat the experiment presented in Figure \ref{fig:revrese_kl} with different values of $\sigma$. The following are plots of the results for $\sigma=0.25, 0.5, 1, 2$.

\begin{figure*}[h!]
\centering
    \centering
    \begin{tabular}{ccc}
      \includegraphics[width=0.23\textwidth, trim={0 0 0 0},clip]{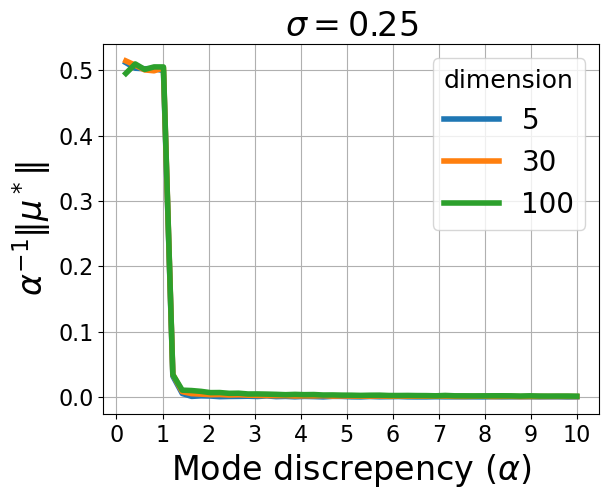} &
      \includegraphics[width=0.23\textwidth, trim={0 0 0 0},clip]{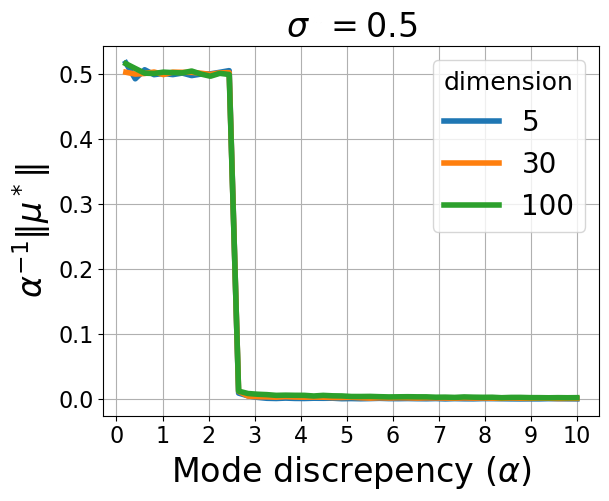} &
      \includegraphics[width=0.23\textwidth, trim={0 0 0 0},clip]{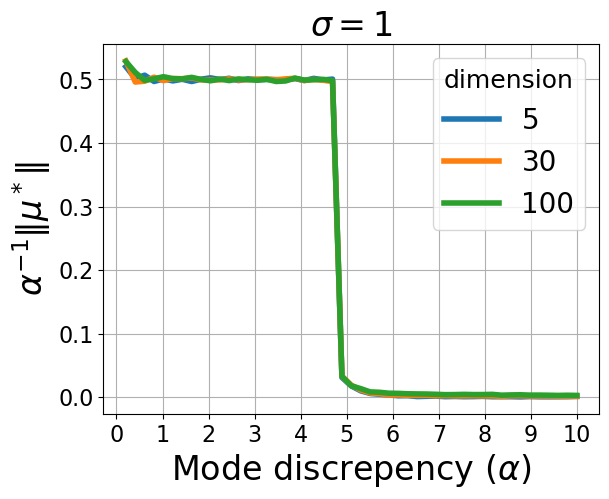} 
      \includegraphics[width=0.23\textwidth, trim={0 0 0 0},clip]{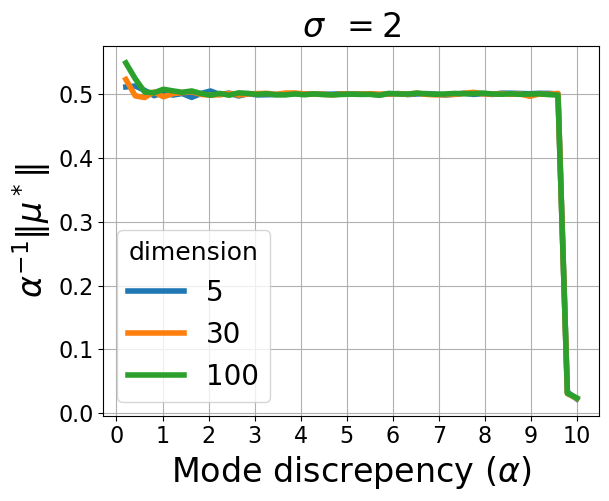} 
    \end{tabular}
    \begin{tabular}{ccc}
      \includegraphics[width=0.23\textwidth, trim={0 0 0 0},clip]{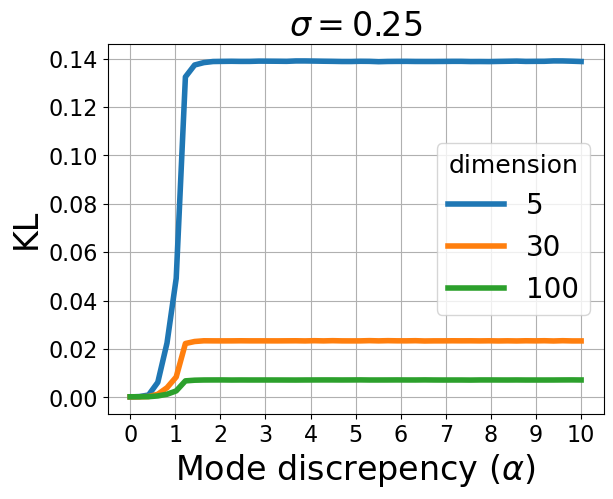} &
      \includegraphics[width=0.23\textwidth, trim={0 0 0 0},clip]{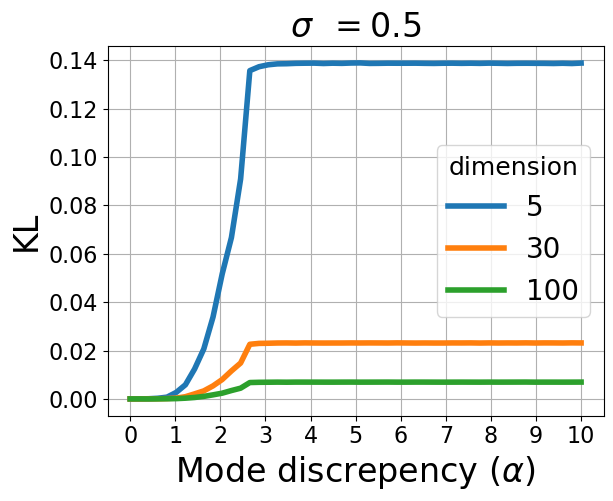} &
      \includegraphics[width=0.23\textwidth, trim={0 0 0 0},clip]{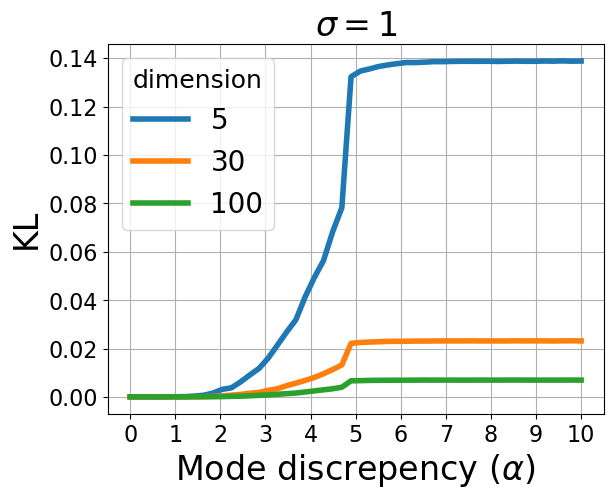} 
      \includegraphics[width=0.23\textwidth, trim={0 0 0 0},clip]{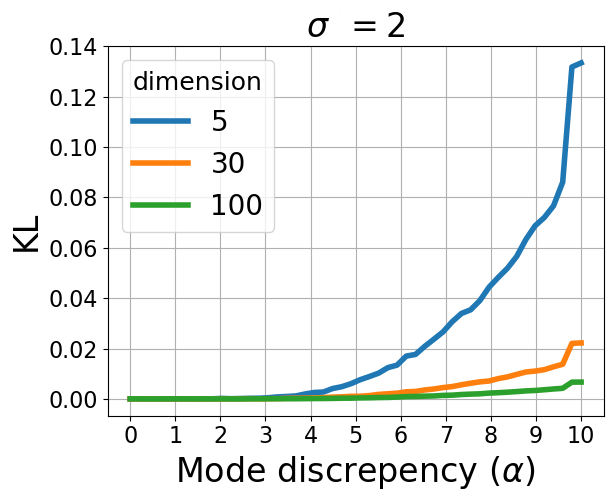} 
    \end{tabular}
    \caption{\textbf{Threshold for mode-seeking behavior.} We plot $\alpha^{-1}\lVert \m^* \rVert_2$ (interpreted as quantifying the amount of interpolation between $p_\alpha^\sigma(\x)$'s modes) and the optimal $\hat{\KL}$ reached by optimization. Plots are provided for $\sigma = 0.25, 0.5, 1, 2$. As can be seen from the plots, the corresponding thresholds $\alpha^* \approx 1.25, 2.5, 5, 10$, scale linearly with $\sigma$.}
\end{figure*}

\section{Mode Proximity Proof} \label{apx:mode_proximity}
\begin{proof}[Proof of Theorem \ref{thm:mode_proximity}]
    Let $\w \neq \mathbf{0}$ be a weight configuration $\in \W$ (the result holds trivially for $\w = \mathbf{0}$). Denote $D = \dim(\W) = d_i d_h + d_h + d_h d_o + d_o$ and fix $R = \frac{D}{\log(\lvert G \rvert)} \lVert \w \rVert_2$. Define $B(\w', r) = \left \{\w \in \W \mid \lVert \w -\w'\rVert_2 \leq r\right\}$ to be the ball of radius $r \in \mathbb{R}$ around $\w' \in \W$. Assume, by way of contradiction, that for all non-trivial $g \in G$, $\lVert \w - g \cdot \w \rVert_2 > 2 R$. Since the action of $G$ is norm-preserving, this implies that $\forall g_1 \neq g_2 \in G$:
    \begin{equation*}
        \lVert g_1 \cdot \w - g_2 \cdot \w \rVert_2 = \lVert g_1 \cdot (\w - g_1^{-1} g_2\cdot w_2)\rVert_2 = \lVert \w - g_1^{-1} g_2\cdot \w \rVert_2 > 2 R
    \end{equation*}
    Therefore, $\forall g_1 \neq g_2 \in G, \, B(g_1 \cdot \w, R) \cap B(g_2 \cdot \w, R) = \emptyset$.
    Additionally, if $\w' \in B(g \cdot \w, R)$ we have
    \begin{equation*}
        \lVert \w' \rVert_2 = \lVert \w' - g \cdot \w + g \cdot \w \rVert_2 \leq \lVert \w' - g\cdot \w \rVert_2 + \lVert g \cdot \w \rVert_2 \leq R + \lVert \w \rVert_2
    \end{equation*}
    so $B(g \cdot \w, R) \subseteq B(\mathbf{0}, \lVert \w \rVert_2 + R)$. Combining these observations, we get that $\bigsqcup_{g\in G} B(g \cdot \w, R)\subseteq B(\mathbf{0}, \lVert \w \rVert_2 + R)$. Therefore, looking at the volume of these subsets
    \begin{equation} \label{eq:vol}
        \sum_{g \in G} \vol(B(g \cdot \w, R)) = \vol \left( \bigsqcup_{g \in G} B(g \cdot \w, R)\right) \leq \vol(B(\mathbf{0}, \lVert \w \rVert_2 + R)) 
    \end{equation}
    The volume of a $D$-dimensional ball of radius $r$ is a constant times $r^D$, so Equation \eqref{eq:vol} becomes
    \begin{equation}\label{eq:basic_bound}
        |G| R^D \leq (\lVert \w \rVert_2 + R)^D
    \end{equation}
    Rearranging and taking the $\log$ of both sides gives $\log(\lvert G \rvert) \leq D \log \left(\frac{\lVert \w \rVert_2}{R} + 1\right)$. Using the fact that $\log(1 + x) < x$ for $x \neq 0$, we get $R < \frac{D}{\log(|G|)} \lVert \w \rVert_2$, which contradicts the definition of $R$. Therefore, the contradicting hypothesis is false, and $\exists g \neq e$ such that $\lVert \w - g \cdot \w \rVert_2 \leq 2 R$. In order to analyze the resulting bound, recall that $D = d_i d_h + d_h + d_h d_o + d_o$ and $\log(|G|) = \log(d_h!) \geq \frac{d_h}{2} \log(\frac{d_h}{2})$. This results in the bound
    \begin{equation}
        \lVert \w - g \cdot \w \rVert_2 \leq 2 R = \frac{2\left(d_i + d_o + 1 + \frac{1}{d_h}\right)}{\log(\frac{d_h}{2})} \lVert \w \rVert_2 = \mathcal{O}\left(\frac{d_i + d_o}{\log(d_h)} \lVert \w \rVert_2\right)
    \end{equation}
    as in Theorem \ref{thm:mode_proximity}.
\end{proof}

\section{Sampling From $q_\theta(\w)$ vs. Sampling From $q_\theta^G(\w)$} \label{apx:sampling}
In Section \ref{sec:symmetrization} we introduced the symmetrized variational posterior $q_\theta^G(\w)$. We mention that we can sample from $q_\theta^G(\w)$ by sampling $\w \sim q_\theta (\w)$ and transforming it using a random $g \sim \text{Uniform}(G)$, and that the induced function space distribution is the same as the one induced by sampling weights from $q_\theta(\w)$. The following elaborates why this is the case. 

Since $q_\theta^G(\w)$ is a uniform mixture of the distributions $\{g_\# q_\theta(\w)\}_{g \in G}$, we can sample from it by sampling $g \sim \text{uniform}(G)$ and then taking $\w \sim g_\# q_\theta(\w)$. By the definition of the pushforward measure, if $\w \sim q_\theta(\w)$, then $\w' = g \cdot \w \sim g_\# q_\theta(\w')$. Therefore, we can sample from $q_\theta^G(\w')$ by independently sampling $g \sim \text{uniform}(G)$, $\w \sim q_\theta(\w)$ and taking $\w' = g \cdot \w$. We now consider the functions $\f^{\w}(\x)$ and $\f^{\w'}(\x)$, where $\w \sim q_\theta(\w)$ and $\w' = g \cdot \w \sim q_\theta^G(\w')$. Since the mapping $\w \mapsto \f^{\w}$ is $G$-invariant (as mentioned in Section \ref{sec:perlimiaries}), the function represented by $\w'$ is $\f^{\w'} \equiv \f^{g \cdot \w} \equiv \f^{\w}$. In other words the function space distribution induced by $\w \sim q_\theta(\w)$ and $\w' \sim q_\theta^G(\w')$ are identical.

\section{Symmetrization of the ELBO} \label{apx:elbo_correction}
In Section \ref{sec:symmetrization} we introduced the symmetrized ELBO objective $\loss_\VI^G(\theta)$. Theorem \ref{thm:elbo_correction} establishes a relationship between $\loss_\VI^G(\theta)$ and $\loss_\VI(\theta)$ and states that
\begin{enumerate}
    \item $\loss_\VI^G(\theta) = \loss_\VI(\theta) + H(q_\theta^G(\w)) - H(q_\theta(w))$. \label{item:second}
    \item $H(q_\theta^G(\w)) - H(q_\theta(\w)) = I(g; \w)$, where $I(g; \w)$ is the mutual information of $g, \w \sim q_\theta^G(g, \w)$. \label{item:third}
\end{enumerate}
\begin{proof}[Proof of Theorem \ref{thm:elbo_correction}]
    Recall that
    \begin{equation}
        \loss_\VI^G(\theta) = \underbrace{\E_{\w \sim q_\theta^G(\w)}\left(\log\left(p(\Y \mid \X, \w)\right)\right)}_\text{expected log-likelihood term} - \underbrace{\KL\left(q_\theta^G(\w) \mid \mid p(\w)\right)}_\text{prior KL term}
    \end{equation}
    We start by analyzing the expected log-likelihood term.
    \begin{equation}
    \begin{split}
        \E_{\w \sim q_\theta^G(\w)} \left(\log(p(\Y \mid \X, \w))\right) &=\int_\W \log(p(\Y \mid \X, \w)) \frac{1}{|G|}\sum_{g \in G} q_\theta(\g^{-1} \cdot \w) \, d\w  \\
        &= \frac{1}{|G|}\sum_{g \in G}\int_\W \log(p(\Y \mid \X, \w)) q_\theta(g^{-1} \cdot \w) \, d\w \\
        &= \frac{1}{|G|}\sum_{g \in G}\int_{g^{-1} \cdot \W} \log(p(\Y \mid \X, \g \cdot \w)) q_\theta(\w) \lvert \det(g) \rvert \, d\w \\
    \end{split}
    \end{equation}
    As mentioned in the discussion leading to Proposition \ref{thm:invariance_of_posterior}, $p\left(\Y \mid \X, \w\right)$ is $G$-invariant. Additionally, $\lvert \det(g) \rvert  = 1$ from the orthonormality of the representation, and $g^{-1} \cdot \W = \W$ since $g^{-1}$ is invertible. Therefore,
    \begin{equation}
    \begin{split}
        \E_{\w \sim q_\theta^G(\w)} \left(\log(p(\Y \mid \X, \w))\right) &= \frac{1}{|G|} \sum_{g \in G}\int_\W \log(p(\Y \mid \X, \w)) q_\theta(\w) \, d\w  \\
        &= \int_\W \log(p(\Y \mid \X, \w)) q_\theta(\w) \, d\w  \\
        &= \E_{\w \sim q_\theta(\w)} \left(\log(p(\Y \mid \X, \w)) \right)
    \end{split}
    \end{equation}
    Thus, $\loss_\VI^G(\theta)$ and $\loss_\VI(\theta)$ have identical expected log-likelihood terms. Moving on to the prior KL term
    \begin{align*}
        \KL(q_\theta^G(\w) \mid \mid p(\w)) &= \E_{\w \sim q_\theta^G(\w)} \left(\log\left(\frac{q_\theta^G(\w)}{p(\w)}\right)\right) \\
        &= \E_{\w \sim q_\theta^G(\w)} \left(-\log\left(p(\w)\right)\right) - H(q_\theta^G(\w)) \\
        &= -\int \log(p(\w)) \frac{1}{|G|}\sum_{g \in G} q_\theta(g^{-1} \cdot \w) \, d\w - H(q_\theta^G(\w)) \\
        &= -\frac{1}{|G|}\sum_{g \in G}\int \log(p(\w)) q_\theta(g^{-1} \cdot \w) \, d\w - H(q_\theta^G(\w)) \\
        &= -\frac{1}{|G|}\sum_{g \in G}\int \log(p(g \cdot \w)) q_\theta(\w) \lvert \det(g) \rvert \, d\w - H(q_\theta^G(\w)) \\
        &= -\frac{1}{|G|}\sum_{g \in G}\int \log(p(\w)) q_\theta(\w) \, d\w - H(q_\theta^G(\w)) \\
        &= -\int \log(p(\w)) q_\theta(\w) \, d\w - H(q_\theta^G(\w)) \\
        &= \E_{\w \sim q_\theta(\w)} \left(-\log(p(\w))\right)  - H(q_\theta^G(\w)) \\
    \end{align*}
    Therefore, since the prior KL term of $\loss_\VI(\theta)$ is $\KL(q_\theta(\w) \mid \mid p(\w)) =E_{\w \sim q_\theta(\w)} \left(-\log\left(p(\w)\right)\right) - H(q_\theta(\w))$, we have
    \begin{equation}
        \KL(q_\theta^G(\w) \mid \mid p(\w)) = \KL(q_\theta(\w) \mid \mid p(\w)) - H(q_\theta^G(\w)) + H(q_\theta(\w))
    \end{equation}
    Putting it all together, we get
    \begin{equation}
        \loss_\VI^G(\theta) = \loss_\VI(\theta) + H(q_\theta^G(\w)) - H(q_\theta(\w))
    \end{equation}
    which finishes the proof of statement \ref{item:second}. All that is left, is to show that indeed $I(g; \w) = H(q_\theta^G(\w)) - H(q_\theta(\w))$. In order to do so, we need the following lemma 
    \begin{lemma}\label{thm:entropy_preservation}
        Let $\mathcal{V}$ be a finite dimensional inner product space. If $p(\mathbf{v})$ is a probability measure on $\V$ and $T \in O(\V)$ is a norm-preserving linear operator on $\V$, then $H(T_\# p(\mathbf{v})) = H(p(\mathbf{v}))$.
    \end{lemma}
    Lemma \ref{thm:entropy_preservation} is proved later in Appendix \ref{apx:elbo_correction}. Recall that $G$ acts on $\W$ via an orthonormal representation and that that $q_\theta^G(\w \mid g) = g_\# q_\theta(\w)$. Therefore, by the definition of mutual information, and using Lemma~\ref{thm:entropy_preservation}, we get
\begin{align*}
    I(g; \w) 
    &= \E_{g, \w \sim q_\theta^G(g, \w)}\left(\log\left(\frac{q_\theta^G(\w \mid g)}{q_\theta^G(\w)}\right)\right) \\
    &= \E_{g, \w \sim q_\theta^G(g, \w)}\left(\log\left(q_\theta^G(\w \mid g)\right)\right) - \E_{g, \w \sim q_\theta^G(g, \w)}\left(\log\left(q_\theta^G(\w)\right)\right) \\
    &= \E_{g \sim q_\theta^G(g)} \left(\E_{\w \sim q_\theta^G(\w \mid g)}\left(\log\left(q_\theta^G(\w \mid g)\right)\right)\right) - \E_{\w \sim q_\theta^G(\w)}\left(\E_{g \sim q_\theta^G(g \mid \w)}\left( \log\left(q_\theta^G(\w)\right)\right)\right) \\
    &= -\E_{g \sim q_\theta^G(g)} \left(H(q_\theta^G(\w \mid g)\right)) - \E_{\w \sim q_\theta^G(\w)}\left( \log\left(q_\theta^G(\w)\right)\right) \\
    &= -\E_{g \sim q_\theta^G(g)} \left(H(g_\# q_\theta(\w)\right)) + H(q_\theta^G(\w))\\ 
    &= -H(q_\theta(\w)) + H(q_\theta^G(\w))
\end{align*}
\end{proof}
All that is left is to prove Lemma \ref{thm:entropy_preservation}.

\begin{proof}[Proof of Lemma \ref{thm:entropy_preservation}]
    With the notations of lemma \ref{thm:entropy_preservation}, and using the change of variable formula for pushforward measures:
    \begin{align*}
        H(T_\# p(\mathbf{v})) &= \int_{\V} p\left(T^{-1}(\mathbf{v})\right) \log\left(p\left(T^{-1}(\mathbf{v})\right)\right) \, d\mathbf{v} \\
        &= \int_{T(\V)} p(\mathbf{v}) \log\left(p(\mathbf{v})\right) \lvert \det(D_{\mathbf{v}} T(\mathbf{v}))\rvert  \, d\mathbf{v} \\
        &= \int_{\V} p(\mathbf{v}) \log\left(p(\mathbf{v})\right) \lvert \det(T)\rvert  \, d\mathbf{v} \\
        &= \int_\V p(\mathbf{v}) \log\left(p(\mathbf{v})\right) \, d\mathbf{v} \\
        &= H(p(\mathbf{v}))
    \end{align*}
    where we used the fact that $T \in O(\V) \Rightarrow T(\mathcal{V}) = \mathcal{V}$ and $\det(T) = \pm 1$.
\end{proof}

\section{Mutual Information Estimation} \label{apx:mi_estimation}
Mutual information (MI) is a fundamental concept in information theory that measures the dependence between two random variables. It plays a crucial role in various areas of machine learning, including representation learning, reinforcement learning, deep learning theory, and more. Given random variables $\x, \y$, with a joint probability distribution $p(\x, \y)$, their mutual information is defined as:
\begin{equation}
    I(\x; \y) = \E_{\x, \y  \sim p(\x, \y)} \left(\frac{p(\x \mid \y)}{p(\x)}\right) = \E_{\x, \y  \sim p(\x, \y)} \left(\frac{p(\y \mid \x)}{p(\y)}\right)
\end{equation}
Depending on the problem, estimating MI directly can be challenging. Therefore, various variational bounds have been established to approximate MI. See \cite{Poole2019OnVB} for a survey. In our case, in order to estimate $I(g; \w)$, we leverage the fact that conditional $q_\theta^G(\w \mid \g) = q_\theta(g^{-1} \cdot \w)$ is tractable, and use the InfoNCE bound.
\begin{theorem}[InfoNCE bound with optimal critic \cite{Poole2019OnVB}]\label{thm:infonce}
    Given two random variables $\x, \y$ distributed by a joint probability function $p(\x, \y)$ the mutual information $I(\x ; \y)$ satisfies 
    \begin{equation}
        I(\x; \y) \geq I^K(\x; \y) = \E\left(\frac{1}{K}\sum_{i = 1}^K\log\left(\frac{p(\x_i \mid \y_i)}{\frac{1}{K}\sum_{j = 1}^K p(\x_i \mid \y_j)}\right)\right) \label{eq:infonce_l}
    \end{equation}
    Where the expectation is over $K$ i.i.d samples from $p(\x, \y)$, i.e. over $\Pi_{i = 1}^K p(\x_i, \y_i)$. Moreover, $\lim_{K \to \infty} I^K(\x ; \y) = I(\x ; \y)$.
\end{theorem}
Proof for $\lim_{K \to \infty} I^K(\x, \y) = I(\x; \y)$ can be found in \cite{Chen2021SimplerFS}. The InfoNCE bound offers a good trade-off between bias and variance compared to other MI estimation methods \cite{Song2020Understanding}. However, reducing bias requires increasing the number of samples $K$, which has quadratic cost in $K$. Fortunately, in our case, the specific structure of $q_\theta^G(\w)$ allows us to simplify Equation \eqref{eq:infonce_l} and obtain a more efficient estimator, presented in Theorem \ref{thm:entropy_bound}.
\begin{proof}[Proof of Theorem \ref{thm:entropy_bound}]
    Plugging in $q_\theta^G(g, \w)$ to Theorem \ref{thm:infonce} yields
        \begin{equation} \label{eq:infonce_gw}
            I^K(\theta) = I^K(g; \w) = \E\left(\frac{1}{K}\sum_{i = 1}^K\log\left(\frac{q_\theta^G(\w_i \mid \g_i)}{\frac{1}{K}\sum_{j = 1}^K q_\theta^G(\w_i \mid \g_j)}\right)\right)
        \end{equation}
    where the expectation is over $g_1, \dots, g_K \sim \text{uniform}(G)$ i.i.d and $\w_i \sim q_\theta^G(\w \mid g_i)$. Note that we can sample $\w_i \sim q_\theta^G(\w \mid g_i)$ by sampling $\w_i' \sim q_\theta(\w)$ and taking $\w_i = g_i \cdot \w_i'$. Using this fact, Equation \eqref{eq:infonce_gw} becomes
    \begin{align}
        \begin{split}
            I^K(\theta) &= \E\left(\frac{1}{K}\sum_{i = 1}^K\log\left(\frac{q_\theta(\w_i')}{\frac{1}{K}\sum_{j = 1}^K q_\theta(g_j^{-1} g_i \cdot \w_i')}\right)\right) \\
            &= -\E\left(\frac{1}{K}\sum_{i = 1}^K\log\left(\frac{1}{K}\sum_{j = 1}^K q_\theta(g_j^{-1} g_i \cdot \w_i')\right)\right) -H(q_\theta(\w)) 
        \end{split}
    \end{align}
    where the expectation is over $g_1, \dots, g_K \sim \text{uniform}(G)$ i.i.d and $\w_1', \dots, \w_K' \sim q_\theta(\w)$ i.i.d. We can further simplify $I^K(\theta)$ using the group structure of $G$, to that end we write 
    \begin{align}
        \begin{split}
            I^K(\theta) &= -H(q_\theta(\w)) - \frac{1}{K}\sum_{i = 1}^K\E_{\w_i', g_1, \dots, g_K}\left(\log \left(\frac{1}{K}\left(q_\theta(\w_i') + \sum_{j \neq i}^K q_\theta(g_j^{-1} g_i \cdot \w_i')\right)\right)\right) \\
            &= -H(q_\theta(\w)) - \frac{1}{K}\sum_{i = 1}^K\E_{g_i}\left(\overbrace{\E_{\w_i', g_1, \dots, g_{i - 1}, g_{i + 1}, \dots, g_K}\left(\log \left(\frac{1}{K}\left(q_\theta(\w_i') + \sum_{j \neq i} q_\theta(g_j^{-1} g_i \cdot \w_i')\right)\right)\right)}^{A(g_i)}\right)
        \end{split}
    \end{align}
    Since $G$ is a group, inverses and right multiplication are bijective transformations. This implies that for a given $i \in \{1, \dots, K\}$, the random variables $\{g_j^{-1} g_i \mid j \neq i\}$ are i.i.d uniformly on $G$ conditioned on $g_i$. Therefore, $A(g_i)$ can be written as
    \begin{equation}
        A(g_i) = \E\left(\log\left(\frac{1}{K}\left(q_\theta(\w_i') + \sum_{j = 1}^{K -1} q_\theta(g_j' \cdot \w_i')\right)\right)\right)
    \end{equation}
    where the expectation is over $\w_i' \sim q_\theta(\w)$ and $g_1', \dots, g_{K - 1}' \sim \text{uniform}(G)$ i.i.d. Notice that $A(g_i)$ is independent of $g_i$! Therefore $I^K(\theta)$ simplifies to
    \begin{equation}
        I^K(\theta) = -H(q_\theta(\w)) -\E\left(\log\left(\frac{1}{K}\left(q_\theta(\w) + \sum_{j = 1}^{K-1} q_\theta(g_j \cdot \w)\right)\right)\right)
    \end{equation}
    Where the expectation is over $\w \sim q_\theta(\w)$ and $g_1, \dots, g_{K - 1} \sim \text{uniform}(G)$. Again using the fact that inversion in bijective in groups, this can be rewritten as in Theorem \ref{thm:entropy_bound}
    \begin{equation}
        I^K(\theta) = -H(q_\theta(\w)) -\E\left(\log\left(\frac{1}{K}\left(q_\theta(\w) + \sum_{j = 1}^{K-1} q_\theta(g_j^{-1} \cdot \w)\right)\right)\right)
    \end{equation}
    Theorem \ref{thm:entropy_bound}, which is phrased in terms of $H^K(\theta)$, can be recovered from the above, using the fact that $I(g; \w) = H(q_\theta^G(\w)) - H(q_\theta(\w))$.
\end{proof}

\end{document}